%% file: unified_sgd_arxiv_fig.tex
\documentclass[10pt]{article}

 \usepackage{authblk}    % use for arXiv
\usepackage{geometry}   % use for arXiv
\usepackage[utf8]{inputenc} % allow utf-8 input
\usepackage{hyperref}       % hyperlinks
\usepackage{url}            % simple URL typesetting
\usepackage{booktabs}       % professional-quality tables
\usepackage{amsfonts}       % blackboard math symbols
\usepackage{nicefrac}       % compact symbols for 1/2, etc.
\usepackage{microtype}      % microtypography
\usepackage{wrapfig}

\usepackage{amsmath,amsthm,amssymb,mathtools,graphicx,enumitem,hyphenat,float}
\usepackage{bbm}
\usepackage{subcaption}
\usepackage{accents, hyperref}
\usepackage{import}
\usepackage{algorithm}
\usepackage{algpseudocode}
\usepackage[sort]{natbib}
\setcitestyle{numbers,open={[},close={]}}
\title{Formatting Instructions For NeurIPS 2020}
\graphicspath{}
\setcitestyle{citesep={,}}

\usepackage{appendix}

\newcommand{\printfnsymbol}[1]{%
  \textsuperscript{\@fnsymbol{#1}}%
}

\input{header.tex}

\begin{document}

\title{\bf Unified Analysis of Stochastic Gradient Methods \\ \bf for Composite Convex and Smooth Optimization}

\author[1,2]{A.\ Khaled\thanks{Equal contribution. Work done during an internship at KAUST.}}
\author[1]{O.\ Sebbouh\thanks{Equal contribution. Work done during an internship at KAUST.}}
\author[3]{N.\ Loizou}
\author[4]{R.\ M.\ Gower}
\author[1]{P.\ Richt\'{a}rik}

\affil[1]{KAUST, Thuwal, Saudi Arabia} % King Abdullah University of Science and Technology
\affil[2]{Cairo University, Giza, Egypt}
\affil[3]{Mila, Universit\'{e} de Montr\'{e}al, Montr\'{e}al, Canada}
\affil[4]{Facebook AI Research, New York, USA}

% NeurIPS style author list
%\author{Ahmed Khaled\thanks{Equal contribution. Work of both first authors was done during an internship at KAUST.} \\
%Cairo University\\
%Giza, Egypt
%\And 
%Othmane Sebbouh\footnotemark[1]\\
%KAUST\\
%Thuwal, Saudi Arabia
%\And Nicolas Loizou\\
%Mila, Université de Montréal\\
%Montreal, Canada
%\And Robert M.\ Gower\\
%Facebook AI Research\\
%New York, USA
%\And Peter Richt\'arik \\
%KAUST\\
%Thuwal, Saudi Arabia
%}

\maketitle

\begin{abstract}
We present a unified theorem for the convergence analysis of stochastic gradient algorithms 
for minimizing a smooth and convex loss plus a convex regularizer. 
We do this by extending the unified analysis of Gorbunov, Hanzely \& Richt\'arik (2020) and dropping the requirement that the loss function be strongly convex. Instead, we only rely on convexity of the loss function.
 Our unified analysis applies to a host of existing algorithms such as proximal SGD, variance reduced methods, quantization and some coordinate descent type methods. For the variance reduced methods, we recover the  best known convergence rates as special cases.
 For proximal SGD, the quantization and coordinate type methods, we uncover new state-of-the-art convergence rates. Our analysis also includes any form of sampling and minibatching.  As such, we are able to determine the minibatch size that optimizes the total complexity of variance reduced methods. We showcase this by obtaining a simple formula for the optimal minibatch size of two variance reduced methods (\textit{L-SVRG} and \textit{SAGA}). This optimal minibatch size not only improves the theoretical total complexity of the methods but also improves their convergence in practice, as we show in several experiments.

% and show that by properly choosing the minibatch size, their convergence speed can be improved both in theory and in practice.
    %In this work we prevent a novel unified analysis for stochastic gradient methods for solving composite finite-sum convex optimization problems. Our analysis allows us to cover a variety of existing methods and techniques, including variance-reduction, quantization, and importance sampling, and we obtain several new rates. Specializing our results to the so-called vanilla SGD setting, we derive guidelines for importance sampling and minibatching.
\end{abstract}

%\tableofcontents

\section{Introduction and Background}\label{sec:intro-background}
Consider the following composite convex optimization problem
\begin{eqnarray}\label{eq:optimization_problem}
\min_{x\in \R^d} \pbr{  F(x) \equiv f(x) + R(x) },
\end{eqnarray}
where $f$ is smooth and convex and $R$ is convex with an easy to compute proximal term. This problem often arises in training machine learning models, where $f$ is a loss function and $R$ is a regularization term, \textit{e.g.} $\ell_1$-reguralized logistic regression \cite{Ravikumar10}, LASSO regression~\cite{Tibshirani1996} and Elastic Net regression \cite{Zou05}. 

A natural algorithm which is well-suited for solving \eqref{eq:optimization_problem} is proximal gradient descent, which requires iteratively taking a proximal step in the direction of the steepest descent. Unfortunately, this method requires computing the gradient $\nabla f$ at each iteration, which can be computationally expensive or even impossible in several settings. This has sparked interest in developing cheaper, practical methods that need only a stochastic unbiased estimate $g_k \in \R^d$ of the gradient at each iteration.
% with provable convergence guarantees, relying on using a gradient descent step with a cheaper, unbiased estimate of the gradient $g_k 
These methods can be written as
\begin{equation}
    \label{eq:sgd-proximal-iterates}
    x_0 \in \R^d, \quad x_{k+1} = \prox_{\gamma_k R} \br{ x_k - \gamma_k g_k },
\end{equation}
where $\br{\gamma_k}_k$ is a sequence of step sizes.  This estimate $g_k$
can take on many different forms depending on the problem of interest. Here we list a few.
% Different choices for $g_k$ aim at solving different problems which arise in machine learning.
%Next, we describe some settings where computing the full gradient is impractical or impossible, and where stochastic gradient methods of the form \eqref{eq:sgd-proximal-iterates} are particularly useful.

\paragraph{Stochastic approximation.} Most machine learning problems can be cast as minimizing the generalization error of some underlying model where $f_z(x)$ is the loss over a sample $z$ and
\begin{eqnarray}\label{eq:stochastic-problem}
f(x) = \ec[z \sim \D]{f_{z}(x)}.
\end{eqnarray}
Since $\D$ is an unknown distribution, computing this expectation is impossible in general. However, by sampling $z \sim \D$, we can compute a \textit{stochastic gradient} $\nabla f_z(x)$. Using Algorithm \eqref{eq:sgd-proximal-iterates} with $g_k = \nabla f_{z_k}(x_k)$ and $R \equiv 0$ gives the simplest stochastic gradient descent method: SGD \cite{Robbins1951, Nemirovski09}.

\paragraph{Finite-sum minimization.} Since the expectation \eqref{eq:stochastic-problem} cannot be computed in general, one well-studied solution to approximately solve this problem is to use a Monte-Carlo estimator:
\begin{eqnarray}\label{eq:finite-sum}
f(x) = \frac{1}{n}\sum_{i=1}^n f_i(x),
\end{eqnarray}
where $n$ is the number of samples and $f_i(x)$ is the loss at $x$ on the $i$-th drawn sample. When $R$ is a regularization function, problem \eqref{eq:optimization_problem} with $f$ defined in \eqref{eq:finite-sum} is often referred to as Regularized Empirical Minimization (R-ERM) \cite{Shwartz14}. For the approximation \eqref{eq:finite-sum} to be accurate, we would like $n$ to be as large as possible. This in turn makes computing the gradient extremely costly. In this setting, for low precision problems, SGD scales very favourably compared to Gradient Descent, since an iteration of SGD requires $\mathcal{O}(d)$ flops compared to $\mathcal{O}(nd)$ for Gradient Descent. Moreover, several techniques applied to SGD such as importance sampling and minibatching \cite{Gower19, Zhao15, Needell14, Konecny16} have made SGD the preferred choice for solving Problem \eqref{eq:optimization_problem} + \eqref{eq:finite-sum}. However, one major drawback of SGD is that, using a fixed step size, SGD does not converge and oscillates in the neighborhood of a minimizer. To remedy this problem, \textit{variance reduced methods} \cite{Schmidt17, Defazio14, Johnson13, Nguyen17, Allen-Zhu17} were developed. These algorithms get the best of both worlds: the global convergence properties of GD and the small iteration complexity of SGD. In the smooth case, they all share the distinguishing property that the variance of their stochastic gradients $g_k$ converges to $0$. This feature allows them to converge to a minimizer with a fixed step size at the cost of some extra storage or computations compared to SGD.

\paragraph{Distributed optimization.} Another setting where the exact gradient $\nabla f$ is impossible to compute is in distributed optimization. The objective function in distributed optimization can be formulated exactly as \eqref{eq:finite-sum}, where each $f_i$ is a loss on the data stored on the $i$-th node. Each node computes the loss on its local data, then the losses are aggregated by the master node. When the number of nodes $n$ is high, the bottleneck of the optimization becomes the cost of communicating the individual gradients. To remedy this issue, various compression techniques were proposed \cite{Seide14, Gupta15, Zhang17, Konecny16randomized, Alistarh18, Wangni18, Alistarh17}, most of which can be modeled as applying a random transformation $Q: \R^d \mapsto \R^d$ to each gradient $\nabla f_i(x_k)$ or to a  noisy estimate of the gradient $g_i^k$. Thus, many proximal quantized stochastic gradient methods fit the form \eqref{eq:sgd-proximal-iterates} with $$g_k = \sum_{i=1}^n Q(g_i^k).$$ While quantized stochastic gradient methods have been widely used in machine learning applications, it was not until the \textit{DIANA} algorithm \cite{Mishchenko19, Mishchenko19waste} that a distributed method was shown to converge to the neighborhood of a minimizer for strongly convex functions. Moreover, in the case where each $f_i$ is itself  a finite average of local functions, variance reduced versions of \textit{DIANA}, called \textit{VR-DIANA} \cite{Horvath19stochastic}, were recently developed and proved to converge sublinearly with a fixed step size for convex functions.

\paragraph{High-dimensional function minimization.} Lastly, regardless of the structure of $f$, if the dimension of the problem $d$ is very high, it is sometimes impossible to compute or to store the gradient at any iteration. Instead, in some cases, one can efficiently compute some coordinates of the gradient, and perform a gradient descent step on the selected coordinates only. These methods are known as (Randomized) Coordinate Descent (RCD) methods \cite{Nesterov12, Wright15}. These methods also fit the form \eqref{eq:sgd-proximal-iterates}, for example with $$g_k = \nabla f(x_k)e_{i_k},$$ where $\br{e_{i}}_i$ is the canonical basis of $\R^d$ and $i_k \in [d]$ is sampled randomly at each iteration. Though RCD methods fit the form~\eqref{eq:sgd-proximal-iterates}  their analysis is often very different compared to other stochastic gradient methods. One exception to this observation is \textit{SEGA} \cite{Hanzely18}, the first RCD method known to converge for strongly convex functions with nonseparable regularizers.

While all the methods presented above have been discovered and analyzed independently, most of them rely on the same assumptions and share a similar analysis. It is this observation and the results derived for strongly convex functions in \cite{Gorbunov19} that motivate this work.

\section{Contributions}\label{sec:contributions}

We now summarize the key contributions of this paper.

\paragraph{Unified analysis of stochastic gradient algorithms.} 
Under a unified assumption on the gradients $g_k$, it was shown in\cite{Gorbunov19}  that Stochastic Gradient methods which fit the format \eqref{eq:sgd-proximal-iterates} converge linearly to a neighborhood of the minimizer for quasi-strongly convex functions when using a fixed step size. We extend this line of work to the convex setting, and further generalize it by allowing for decreasing step sizes. As a result, for all the methods which verify our assumptions, we are able to prove either sublinear  convergence to the neighborhood of a minimum with a fixed step size or exact convergence with a decreasing step size.

\paragraph{Analysis of SGD without the bounded gradients assumption.} Most of the existing analysis on SGD assume a uniform bound on the second moments of the stochastic gradients or on their variance. Indeed, for the analysis of Stochastic (sub)gradient descent, this is often necessary to apply the classical convergence proofs. However, for large classes of convex functions, it has been shown that these assumptions do not to hold \cite{Nguyen18, Khaled20}. As a result, there has been a recent surge in trying to avoid these assumptions on the stochastic gradients for several classes of smooth functions: strongly convex \cite{Nguyen18, Grimmer2019, loizou2020stochastic}, convex \cite{Grimmer2019,Stich2019,Vaswani18,loizou2020stochastic}, or even nonconvex functions \cite{Khaled20,Lei19,loizou2020stochastic}. Surprisingly, a general analysis for convex functions without these bounded gradient assumptions is still lacking. As a special case of our unified analysis, assuming only convexity and smoothness, we provide a general analysis of proximal SGD in the convex setting. Moreover, using the \textit{arbitrary sampling} framework \cite{Gower19}, we are able to prove convergence rates for SGD under minibatching, importance sampling, or virtually any form of sampling.

\paragraph{Extension of the analysis of existing algorithms to the convex case.} As another special case of our analysis, we also provide the first convergence rates for the (variance reduced) stochastic coordinate descent method \textit{SEGA} \cite{Hanzely18} and the distributed (variance reduced) compressed SGD method \textit{DIANA} \cite{Mishchenko19} in the convex setting. Our results can also be applied to all the recent  methods developed in \cite{Gorbunov19}.

\paragraph{Optimal minibatches for \textit{L-SVRG} and \textit{SAGA} in the convex setting.} 
With a unifying convergence theory in hand, we can now ask sweeping questions across families of algorithms. We demonstrate this by answering the question 
\begin{quote}``What is the optimal minibatch size for variance reduced methods?'' \end{quote}  
%To answer this, first we derive a precise upper bound on the total complexity of \emph{all} variance reduced methods.
%As such, we are now in a position that we can pose 
%As a result, we are able to optimize this total complexity with respect to any hyperparameter. One crucial hyperparameter that often needs to be tuned for Stochastic Gradient methods is the minibatch size. 
Recently, precise estimates of the minibach sizes which minimize the total complexity for \textit{SAGA} \cite{Defazio14} and \textit{SVRG} \cite{Johnson13, Zhu16, Reddi16} applied to strongly convex functions were derived in \cite{Gazagnadou19} and \cite{Sebbouh19}. 
%These optimal minibatch sizes depend on the strong convexity parameter, and thus can be numerically impossible to compute in some cases.
We showcase the flexibility of our unifying framework by deriving new optimal minibatch sizes for \textit{SAGA} \cite{Defazio14} and \textit{L-SVRG} \cite{Hofmann15, Kovalev20} in the general convex setting. Unlike prior work in the strongly convex setting  \cite{Gazagnadou19} and \cite{Sebbouh19}, our resulting optimal minibatch sizes can be computed using only the smoothness constants.
To verify the validity of our claims, we show through extensive experiments that our theoretically derived optimal minibatch sizes are competitive against a gridsearch.
% and the optimal miniblock size for \textit{SEGA}.

\section{Unified Analysis for Proximal Stochastic Gradient Methods}\label{sec:unified-analysis}

\paragraph{Notation.}   The Bregman divergence associated with $f$ is the mapping $$D_{f} (x, y) \; \eqdef \; f(x) - f(y) - \ev{ \nabla f(y), x - y }, \quad x,y\in R^d$$ and the proximal operator of $\gamma R$ is the function $$\prox_{\gamma R} \br{x} \; \eqdef \; \argmin_u \left\{ \gamma R(x) + \frac{1}{2} \norm{x-u}^2\right\}.$$ Let $[n] \eqdef \left\{1,\dots,n\right\}$. 
%For all $Z \in \R^{d\times n}$ let $\trn{Z} \eqdef \tr\br{{ZZ^\top}}$ where $\tr\br{\cdot}$ is the trace.

In \cite{Gorbunov19}, Stochastic Gradient methods that fit the form \eqref{eq:sgd-proximal-iterates} were analyzed for smooth quasi-strongly convex functions. In this work, we extend these results to the general convex setting. We  formalize our assumptions on $f$ and $R$ in the following.
\begin{assumption} \label{asm:function-class}
The function $f$ is  $L$--smooth and convex:
\begin{align}
f(y) & \leq f(x) + \langle \nabla f(x), y - x \rangle + \frac{L}{2} \norm{y-x}^2,  \quad \mbox{for all }x,y \in \R^d, \label{eq:Lsmooth}\\
f(y )& \geq f(x) + \langle \nabla f(x), y - x \rangle, \quad \mbox{for all }x,y \in \R^d. \label{eq:convex}
\end{align}
The function $R$ is convex:
\begin{equation}\label{eq:R-convex}
R(\alpha x + (1- \alpha) y) \geq \alpha R(x) + (1-\alpha) R(y), \quad \mbox{for all }x,y \in \R^d, \alpha \in [0,\,1].
\end{equation}

\end{assumption}

When $f$ has the form \eqref{eq:finite-sum}, we assume that for all $i \in [n]$, $f_i$ is $L_i$-smooth and convex, and we denote $L_{\max} \eqdef \underset{i\in [n]}{\max}\,L_i$.

The innovation introduced in \cite{Gorbunov19} is the following unifying assumption on the stochastic gradients $g_k$ used in \eqref{eq:sgd-proximal-iterates} which allows to simultaneously analyze classical SGD, variance reduced methods, quantized stochastic gradient methods, and some randomized coordinate descent methods.

\begin{assumption}[Assumption 4.1 in \cite{Gorbunov19}]\label{asm:main_assumption}
Consider the iterates $\br{x_k}_k$ and gradients $\br{g_k}_k$ in \eqref{eq:sgd-proximal-iterates}.
    \begin{enumerate}
        \item The gradient estimates are unbiased:
        \begin{equation}
            \label{eq:asm-gradient-unbiased}
            \ec{g_k \mid x_k} = \nabla f(x_k).
        \end{equation}
        \item There exist constants $A,B,C,D_1, D_2, \rho \geq 0$, and a sequence of random variables $\sigma^2_k \geq 0 $  such that:
        \begin{align}
            \label{eq:asm-gradient-opt-distance}
            \ec{ \norm{g_k - \nabla f(x_\ast)}^2 \mid x_k } &\leq 2 A D_{f} (x_k, x_\ast) + B \sigma_k^2 + D_1, \\
            \label{eq:asm-decreasing-noise}
            \ec{ \sigma_{k+1}^2 \mid x_k } &\leq \br{ 1 - \rho } \sigma_k^2 + 2 C D_{f} (x_k, x_\ast) + D_2.
        \end{align}
    \end{enumerate} 
\end{assumption}

Though we chose to present Equations \eqref{eq:asm-gradient-unbiased}, \eqref{eq:asm-gradient-opt-distance} and \eqref{eq:asm-decreasing-noise} as an assumption, we show throughout the main paper and in the appendix that for all the algorithms we consider (excluding \textit{DIANA}), these equations all hold with known constants when Assumption~\ref{asm:function-class} holds. An extensive yet nonexhaustive list of algorithms satisfying Assumption \ref{asm:main_assumption} and the corresponding constants can be found in Table~2 in \cite{Gorbunov19}. We report in Section \ref{sec:app-corollaries-theorem} of the appendix these constants for five algorithms: \textit{SGD}, two variance reduced methods \textit{L-SVRG} and \textit{SAGA}, a distributed method \textit{DIANA} and a coordinate descent type method \textit{SEGA}
%, which paint a fairly large picture of the different algorithms encountered in stochastic gradient optimization.

We now state our main theorem.
\begin{theorem}
    \label{thm:main-prox-dec}
    Suppose that Assumptions~\ref{asm:function-class} and~\ref{asm:main_assumption} hold. Let $M \eqdef B/\rho$ and let $(\gamma_k)_{k\geq 0}$ be a decreasing, strictly positive sequence of step sizes chosen such that
    \[ 0 < \gamma_0 < \min \left\{ \frac{1}{2 (A + MC)}, \frac{1}{L} \right\}. \]
The iterates given by \eqref{eq:sgd-proximal-iterates} satisfy
    \begin{align}
        \label{eq:main-thm-convergence-bound}
        \ec{F(\bar{x}_t) - F(x_\ast)} \leq \frac{\sqn{x_0 - x_*} + 2\gamma_0\br{\delta_0 + \gamma_0M\sigma_0^2} + 2\br{D_1 + 2 M D_2}\sum \limits_{k=0}^{t-1}\gamma_k^2}{2\sum \limits_{i=0}^{t-1}\br{1 - 2\gamma_i\br{A+MC}} \gamma_i},
    \end{align}
    where $\bar{x}_t \eqdef \sum\limits_{k=0}^{t-1} \frac{\br{ 1 - 2\gamma_k(A+MC)}  \gamma_k}{\sum \limits_{i=0}^{t-1}\br{1 - 2\gamma_i(A+MC)} \gamma_i}x_k$ and ${\delta_0 \eqdef F(x_0) - F(x_\ast)}$.
\end{theorem}
The proof of Theorem \ref{thm:main-prox-dec} is deferred to the the appendix (Section \ref{sec:app-unified-analysis-proofs}).

\section{The Main Corollaries}\label{sec:main-corollaries}

In contrast to \cite{Gorbunov19}, our analysis allows both  for constant and decreasing step sizes. In this setion, we will present two corollaries corresponding to these two choices of step sizes and discuss the resulting convergence rates depending on the constants obtained from Assumption \ref{asm:main_assumption}. Then, we specialize our theorem to SGD, which allows us to recover the first analysis of SGD without the bounded gradients or bounded gradient variance assumptions in the general convex setting. We apply the same analysis to \textit{DIANA} and present the first convergence results for this algorithm in the convex setting.

First, we show that by using a constant step size the average of iterates of any stochastic gradient method of the form \eqref{eq:sgd-proximal-iterates} satisfying Assumptions~\ref{asm:function-class} and~\ref{asm:main_assumption} converges sublinearly to the neighborhood of the minimum.

\begin{corollary}
    \label{cor:conv-fixed}
  Consider the setting of Theorem~\ref{thm:main-prox-dec}.  Let $M = B/\rho$. Choose stepsizes $\gamma_k = \gamma > 0$ for all $k$, where ${\gamma \leq \min\left\{ \frac{1}{4 (A + MC)}, \frac{1}{2L} \right\}}$, then substituting in the rate in \eqref{eq:main-thm-convergence-bound} we have,
    \begin{equation}\label{eq:cor_fixed_step}
        \ec{F(\bar{x}_t) - F(x_\ast)} \leq \frac{2\gamma\br{\delta_0 + \gamma M\sigma_0^2}+\sqn{x_0 - x_*}}{\gamma t} + 2 \gamma \br{D_1 + M D_2}.
    \end{equation}
 \end{corollary}

One can already see that to ensure convergence with a fixed step size, we need to have $D_1 = D_2 = 0$. The only known stochastic gradient methods which satisfy this property are variance reduced methods, as we show in Section \ref{sec:optimal-minibatch}. When $D_1 \neq 0$ or $D_2 \neq 0$, which is the case for \textit{SGD} and \textit{DIANA} (See Section \ref{sec:app-corollaries-theorem}), the solution to ensure anytime convergence is to use decreasing step sizes.

\begin{corollary}
	\label{cor:conv-dec}
	  Consider the setting of Theorem~\ref{thm:main-prox-dec}. Let $M = B/\rho$. Choose stepsizes $\gamma_k = \frac{\gamma}{\sqrt{k+1}}$ for all $k \geq 0$, where $\gamma \leq \min\left\{ \frac{1}{4 (A + MC)}, \frac{1}{2L} \right\}$. Then substituting in the rate in \eqref{eq:main-thm-convergence-bound}, we have
\begin{align}
        \ec{F(\bar{x}_t) - F(x_\ast)} &\leq \frac{\gamma\br{\delta_0 + \gamma M\sigma_0^2} +\sqn{x_0 - x_*} + \br{\tfrac{D_1}{2} +  M D_2}\br{\log(t)+1}}{\gamma \br{\sqrt{t} - 1}} \\
        &\sim \mathcal{O}\left(\frac{\log(t)}{\sqrt{t}}\right)
    \end{align}	
\end{corollary}

\subsection{SGD without the bounded gradients assumption}\label{sec:SGD-simple}
To better illustrate the significance of the convergence rates derived in Corollaries \ref{cor:conv-fixed} and \ref{cor:conv-dec},  consider  the SGD method for the finite-sum setting \eqref{eq:finite-sum}:
\begin{equation}\label{eq:sgd-proximal-iterates-vanilla}
 x_0 \in \R^d, \quad  x_{k+1} = \prox_{\gamma_k R} \br{ x_k - \gamma_k \nabla f_{i_k}(x_k)},
\end{equation}
where $i_k$ is sampled uniformly at random from $[n]$.
\begin{lemma}\label{lem:constants_sgd}
Assume that $f$ has a finite sum structure \eqref{eq:finite-sum} and that Assumption \ref{asm:function-class} holds. The iterates defined by \eqref{eq:sgd-proximal-iterates-vanilla} verify Assumption \ref{asm:function-class} with
\begin{eqnarray}\label{eq:constants-SGD-simple}
A = 2L_{\max}, \; B = 0, \; \rho = 1, \; C = 0, \; D_1 = 2\sigma^2, D_2 = 0,
\end{eqnarray}
where $\sigma^2 = \frac{1}{n}\underset{x_* \in X^*}{\sup}\sum\limits_{i=1}^n\sqn{\nabla f_i(x_*)}$ and $L_{\max} = \underset{i\in[n]}{\max} \, L_i$.
\end{lemma}
\begin{proof}
See Lemma A.1 in \cite{Gorbunov19}.
\end{proof}
This analysis can be easily extended to include minibatching, importance sampling, and virtually all forms of sampling by using  the constants given in \eqref{eq:constants-SGD-simple}, with the exception of $L_{\max}$ which should be replaced by the \textit{expected smoothness} constant \cite{Gower19}. Due to lack of space, we defer this general analysis of SGD to the appendix (Sections \ref{sec:app-arbitrary-sampling} and \ref{sec:app-corollaries-theorem}). Using Theorem \ref{thm:main-prox-dec} and Lemma \ref{lem:constants_sgd}  we arrive at the following result.
\begin{corollary}\label{cor:conv-SGD-simple}
Let $(\gamma_k)_k$ be a sequence of decreasing step sizes such that $0<\gamma_0 \leq 1/4L_{\max}$ for all $k \in \N$. Let Assumption~\ref{asm:function-class} hold. %By Theorem \ref{thm:main-prox-dec} and Lemma \ref{lem:constants_sgd} we have that
 The iterates of \eqref{eq:sgd-proximal-iterates-vanilla} verify
\begin{align}
        \label{eq:convergence-bound-sgd-vanilla}
        \ec{F(\bar{x}_t) - F(x_\ast)} \leq \frac{\sqn{x_0 - x_\ast} + 2\gamma_0\br{F(x_0) - F(x_\ast)}}{\sum \limits_{i=0}^{t-1}\gamma_i} + \frac{2\sigma^2 \sum \limits_{k=0}^{t-1}\gamma_k^2}{\sum \limits_{i=0}^{t-1}\gamma_i}.
    \end{align}
\end{corollary}
%\rob{I'm thinking. Maybe it's better to display a particular choice of stepsizes such as the decreasing stepsizes, to make it clear that we have a $\log(t)/t$ convergences? Not sure...}
Moreover, as we did in Corollaries \ref{cor:conv-fixed} and \ref{cor:conv-dec}, we can show sublinear convergence to a neighborhood of the minimum if we use a fixed step size, or $\mathcal{O}(\log(k)/\sqrt{k})$ convergence to the minimum using a step size $\gamma_k = \frac{\gamma}{\sqrt{k+1}}$. Moreover, if we know the stopping time of the algorithm, we can derive a $\mathcal{O}(1/\sqrt{k})$ upper bound as done in \cite{Nemirovski09}.
 
Corollary~\ref{cor:conv-SGD-simple}  fills a gap in the theory of SGD. Indeed, to the best of our knowledge, this is the first analysis of proximal SGD in the convex setting
which does not assume neither bounded gradients  nor bounded variance (as done in \textit{e.g.} \cite{Nemirovski09, Ghadimi13}). Instead, it relies only on convexity and smoothness. The closest results to ours here are Theorem~6 in \cite{Grimmer2019} and Theorem~5 in \cite{Stich2019}, both of which are in the same setting as Lemma~\ref{lem:constants_sgd} but study more restrictive variants of proximal SGD. Grimmer \cite{Grimmer2019} studies SGD with projection onto closed convex sets and Stich \cite{Stich2019} studies vanilla SGD, without proximal or projection operators. Unfortunately, neither result extends easily to include using proximal operators, and hence our results necessitate a different approach. 

% cite{Vaswani18}, where the function is assumed to be convex, smooth and to satisfy the so-called interpolation assumption (corresponding to $\sigma^2 = 0$). Their rate can be recovered as a special case of ours by assuming in addition that $\sigma^2 = 0$ and using a fixed step size in~\eqref{eq:convergence-bound-sgd-vanilla}.

\subsection{Convergence of \textit{DIANA} in the convex setting}
\textit{DIANA} was the first distributed quantized stochastic gradient method proven to converge to the minimizer in the strongly convex case and to a critical point in the nonconvex case \cite{Mishchenko19}. See Section \ref{sec:DIANA-app} in the appendix for the definition of \textit{DIANA} and its parameters.

\begin{lemma}\label{lem:constants-DIANA-main}
Assume that $f$ has a finite sum structure and that Assumption~\ref{asm:function-class} holds. The iterates of \textit{DIANA} (Algorithm \ref{alg:DIANA}) satisfy Assumption \ref{asm:main_assumption} with constants:
\begin{eqnarray}\label{eq:params-DIANA-main}
A = \br{1+\frac{2w}{n}}L_{\max}, \; B = \frac{2w}{n}, \; \rho = \alpha, \; C = L_{\max}\alpha, \; D_1 = \frac{(1+w)\sigma^2}{n}, \;  D_2 = \alpha\sigma^2,
\end{eqnarray}
where $w > 0$ and $\alpha \leq \frac{1}{1+w}$ are parameters of Algorithm \ref{alg:DIANA} and $\sigma^2$ is such that $$\forall k \in \N, \quad \frac{1}{n}\sum_{i=1}^n\ec{\sqn{g_i^k - \nabla f(x_k)}} \leq \sigma^2.$$
\end{lemma}
\begin{proof}
See Lemma A.12 in \cite{Gorbunov19}.
\end{proof}
As yet another corollary of Theorem \ref{thm:main-prox-dec}, we can extend the results of \cite{Mishchenko19} to the convex case and show that \textit{DIANA} converges sublinearly to the neighborhood of the minimum using a fixed step size, or to the minimum exactly using a decreasing step size.
\begin{corollary}
    \label{cor:conv-DIANA-main}
    Assume that $f$ has a finite sum structure \eqref{eq:finite-sum} and that Assumption \ref{asm:function-class} holds. Let $(\gamma_k)_{k\geq 0}$ be a decreasing, strictly positive sequence of step sizes chosen such that
    \[ 0 < \gamma_0 < \frac{1}{4(1 + \frac{4w}{n})L_{\max}}. \]
     By Theorem \ref{thm:main-prox-dec} and Lemma \ref{lem:constants-DIANA-main}, we have that the iterates given by Algorithm \ref{alg:DIANA} verify
    \begin{align}
        \label{eq:conv-DIANA-main}
        \ec{F(\bar{x}_t) - F(x_\ast)} \leq \frac{\sqn{x_0 - x_*} + 2\gamma_0\br{F(x_0) - F(x_*)+ \frac{2w\gamma_0}{\alpha n}\sigma_0^2 } +  \frac{2\br{1+5w}\sigma^2}{n}  \sum \limits_{k=0}^{t-1}\gamma_k^2}{\sum \limits_{i=0}^{t-1}\gamma_i}.
    \end{align}
\end{corollary}

\section{Optimal Minibatch Sizes for Variance Reduced Methods}\label{sec:optimal-minibatch}
Variance reduced methods are of particular interest because they do not require a decreasing step size in order to ensure convergence. 
This is because for variance reduced methods we have $D_1 = D_2 = 0$, and thus, these methods converge sublinearly with a fixed step size.

% a fixed step size directly results in sublinear convergence. 
%Instead, since for these methods $D_1 = D_2 = 0$, using a fixed step size directly results in sublinear convergence. 

The variance reduced methods were designed for solving~\eqref{eq:optimization_problem} in the special case where $f$ has a finite sum structure. In this case, in order to further improve the convergence properties of variance reduced methods, several techniques can be applied such as adding momentum \cite{Allen-Zhu17} or using importance sampling \cite{Gower18}, but the most popular of such techniques is by far minibatching. Minibatching has been used in conjuction with variance reduced methods since their inception \cite{Konecny16}, but it was not until \cite{Gazagnadou19, Sebbouh19} that a theoretical justification for the effectiveness of minibatching was proved for SAGA \cite{Defazio14} and SVRG \cite{Johnson13} in the strongly convex setting. In this section, we show how 
our theory allows us to determine the  optimal minibatch sizes which minimize the total complexity of any variance reduced method.
%\rob{We don't unify these results.}
%can unify the methodologies used in \cite{Gazagnadou19} and \cite{Sebbouh19} to find the optimal minibatch sizes which minimize the total complexity of variance reduced methods. 
This allows us to compute the first estimates of these minibatch sizes in the nonstrongly convex setting.
For simplicity, in the remainder of this section, we will consider the special case where $R = 0$. Hence, in this section
\begin{eqnarray}
F(x) = f(x) \equiv \frac{1}{n}\sum_{i=1}^n f_i(x).
\end{eqnarray}

%More particularly, we will focus on the celebrated SAGA \cite{Defazio14} and SVRG \cite{Johnson13} methods. Two recent works \cite{Gazagnadou19, Sebbouh19} proved that for strongly convex functions, the total complexity of these methods can be significantly reduced by using a minibatch size which depends on the problem setting. Our goal is to examine wether this is also the case in the convex setting.

%Our ultimate goal is to minimize the total complexity of the variance reduced methods. 
To derive a meaningful optimal minibatch size from our theory, we need to use the tightest possible upper bounds on the total complexity. When $R = 0$, we can derive a slightly tighter upper bound than the one we obtained in Theorem \ref{thm:main-prox-dec} as follows.
\begin{proposition}\label{prop:conv-vr-smooth}
Let $R = 0$ and $M=B/2\rho$. Suppose that Assumption \ref{asm:main_assumption} holds with $D_1 = D_2 = 0$. Let the step sizes $\gamma_k = \gamma$ for all $k\in \N$, with $\gamma_k = \gamma \leq 1/(4(A+MC))$ for all $k \in \N$. Then,
    \begin{align}
        \label{eq:thm-convergence-bound-smooth-nice}
        \ec{f(\bar{x}_k) - f(x_\ast)} \leq \frac{\sqn{x_0 - x_*} + 2M\gamma^2\sigma_0^2}{\gamma k}.
    \end{align}
\end{proposition}

We can translate this upper bound into a convenient complexity result as follows.
\begin{corollary}\label{cor:complexity-vr-methods}
Assume that there exists a constant $G \geq 0$ such that
\begin{eqnarray}\label{eq:def_G}
\sigma_0^2 \leq G \, \sqn{x_0 - x_*}.
\end{eqnarray}
Let $\epsilon > 0$ and  $\gamma = \frac{1}{4(A+\frac{BC}{2\rho})}$.
It follows that 
\begin{eqnarray}\label{eq:complexity-vr-methods}
k \geq \br{4(A+ \frac{BC}{2\rho}) + \frac{BG}{2(2\rho A+BC)}}\frac{\sqn{x_0 - x_*}}{\epsilon} \; \implies \; \ec{f(\bar{x}_k) - f(x_*)} \leq \epsilon.
\end{eqnarray}
% If  
%\begin{eqnarray}\label{eq:complexity-vr-methods}
%k \geq \br{4(A+ \frac{BC}{2\rho}) + \frac{BG}{2(2\rho A+BC)}}\frac{\sqn{x_0 - x_*}}{\epsilon},
%\end{eqnarray}
%then $\ec{f(\bar{x}_k) - f(x_*)} \leq \epsilon$.
\end{corollary}
\begin{proof}
The result follows from taking $\gamma = \frac{1}{4(A+\frac{BC}{2\rho})}$ and upperbounding $\sigma_0^2$ by G in \eqref{eq:thm-convergence-bound-smooth-nice}.
\end{proof}

In the same way we specialized the general convergence rate given in Theorem \ref{thm:main-prox-dec} to the cases of \textit{SGD} and \textit{DIANA} in Section \ref{sec:main-corollaries}, we can specialize the iteration complexity result \eqref{eq:complexity-vr-methods} to any method which verifies $D_1 = D_2 = 0$. Due to their popularity, we chose to analyze minibatch variants of \textit{SAGA} \cite{Defazio14} and \textit{L-SVRG} \cite{Hofmann15, Kovalev20} (a single-loop variant of the original SVRG algorithm \cite{Johnson13}). The pseudocode for these algorithms is presented in Algorithms \ref{alg:b-SAGA} and \ref{alg:b-L-SVRG}. We define for any subset $B \subseteq [n]$ the minibatch average of $f$ over $B$ as $f_B(x) = \frac{1}{b}\sum_{i \in B} f_i(x).$

\begin{minipage}{\textwidth}
   \centering
   \begin{minipage}{.45\textwidth}
     %\captionof{algorithm}{$b-$SAGA}
     \begin{algorithm}[H]
     \begin{algorithmic}
     \State \textbf{Parameters} minibatch size $b$, step size $\gamma$
    \State \textbf{Initialization}   $x_0 \in \R^d$ and $J_0^i = \nabla f_i(x_0)$ for $i =1,\ldots, n$.
    % \For {$k=0, 1, 2,\dots$}\vskip 1ex
    \For {$k=0, 1,\dots$}\vskip 1ex
      \State Sample a batch $B \subseteq [n]$ with $|B| = b$
      \State $g_k = \frac{1}{n}\sum_{i=1}^n J_k^i + \nabla f_B(x_k) -  \frac{1}{b}\sum_{i \in B}J_{k}^{i} $ % 	 	\Comment Stochastic Gradient estimate. \label{com:sto_step}
      \State $x_{k+1} = x_k - \gamma g_k$
      \State $ J_{k+1}^{i} = \left\{
      \begin{array}{ll}
          J_{k}^{i} & \mbox{if } i \notin B \\
          \nabla f_i(x_k) & \mbox{if } i \in B
      \end{array} \right.$
    \EndFor
     \end{algorithmic}
     \caption{$b$-SAGA}
  \label{alg:b-SAGA}
     \end{algorithm}
   \end{minipage}
\hspace{0.2cm}
   \begin{minipage}{.45\textwidth}
     \begin{algorithm}[H]
     \begin{algorithmic}
     \State \textbf{Parameters} minibatch size $b$, step size $\gamma$, $p\in (0,1]$
    \State \textbf{Initialization}   $w_0 = x_0 \in \mathbb{R}^d$
    % \For {$k=0, 1, 2,\dots$}\vskip 1ex
    \For {$k=0, 1,\dots$}\vskip 1ex
      \State Sample a batch $B \subseteq [n]$ with $|B| = b$
      \State $g_k =\nabla f_{B}(x_k) - \nabla f_{B}(w_k) + \nabla f(w_k)$ % 	 	\Comment Stochastic Gradient estimate. \label{com:sto_step}
      \State $x_{k+1} = x_k - \gamma g_k$
      \State $ w_{k+1} = \left\{
      \begin{array}{ll}
          x_k & \mbox{w. prob. }p \\
          w_{k} & \mbox{w. prob. } 1-p
      \end{array} \right.$
    \EndFor
     \end{algorithmic}
     \caption{$b$-L-SVRG}
  \label{alg:b-L-SVRG}
   \end{algorithm}
   \end{minipage}
\end{minipage}

\bigskip
As we will show next, the iterates of Algorithms \ref{alg:b-SAGA} and \ref{alg:b-L-SVRG} satisfy Assumption \ref{asm:main_assumption} with constants which depend on the minibatch size $b$. These constants will depend on the following \textit{expected smoothness} and \textit{expected residual} constants $\cL(b)$ and $\zeta(b)$ used in the analysis of \textit{SAGA} and \textit{SVRG} in \cite{Gazagnadou19, Sebbouh19}:
\begin{eqnarray}
\cL(b) &\eqdef& \frac{1}{b}\frac{n-b}{n-1}L_{\max} + \frac{n}{b}\frac{b-1}{n-1}L, 
\quad \mbox{and} \quad
\zeta(b) \;\eqdef\;  \frac{1}{b}\frac{n-b}{n-1}L_{\max}. \label{eq:cL-b-nice}
\end{eqnarray}
%
%
%\begin{eqnarray}
%\cL(b) &\eqdef& \frac{1}{b}\frac{n-b}{n-1}L_{\max} + \frac{n}{b}\frac{b-1}{n-1}L, \label{eq:cL-b-nice} \\
%\zeta(b) &\eqdef&  \frac{1}{b}\frac{n-b}{n-1}L_{\max}.\label{eq:zeta-b-nice}
%\end{eqnarray}

%Due to lack of space, we will present in full detail only the analysis of the optimal minibatch size for \textit{b-SAGA} (Algorithm \ref{alg:b-SAGA}) in the main paper, and defer that of \textit{L-SVRG} (Algorithm \ref{alg:b-L-SVRG}) to the appendix (Section \ref{sec:app-optimal-minibatch-b-L-SVRG}).

\subsection{Optimal minibatch size for SAGA}\label{sec:optimal-b-SAGA}
Consider the $b$-SAGA method in Algorithm \ref{alg:b-SAGA}. Define $$H(x) \eqdef \left[f_1(x),\dots,f_n(x)\right] \in \R^{d}$$ and let $\nabla H(x) \in \R^{d\times n}$ denote the Jacobian of $H$.  Let  $J_k = [J_k^1, \ldots, J_k^n]$ be the current \emph{stochastic Jacobian}.

\begin{lemma}\label{lem:constants-b-SAGA-main}
The iterates of Algorithm \ref{alg:b-SAGA} satisfy Assumption \ref{asm:main_assumption} and Equation \eqref{eq:def_G} with 
\begin{eqnarray}\label{eq:sigma-b-SAGA}
\sigma_k^2 = \frac{1}{nb} \frac{n-b}{n-1}\trn{J_k - \nabla H(x_*)},
\end{eqnarray}
where for all $Z \in \R^{d\times n},$ $\trn{Z} = \tr\br{{ZZ^\top}}$, and constants
\begin{eqnarray}\label{eq:constants-b-SAGA}
A = 2\cL(b), \; B = 2, \; \rho = \frac{b}{n}, \; C = \frac{b\zeta(b)}{n}, \; D_1 = D_2 = 0, \; G = \zeta(b)L.
\end{eqnarray}

\end{lemma}
Using Corolary \ref{cor:complexity-vr-methods}, we can determine the iteration complexity of Algorithm \ref{alg:b-SAGA}.

\begin{corollary}[Iteration complexity of $b-$SAGA]\label{cor:complexity-b-SAGA}
Consider the iterates of Algorithm \ref{alg:b-SAGA}. Let $\gamma = \frac{1}{4(2\cL(b)+\zeta(b))}$. Given the constants obtained for Algorithm \ref{alg:b-SAGA} in \eqref{eq:constants-b-SAGA}, by Corollary \ref{cor:complexity-vr-methods} we have that
\begin{eqnarray*}
k &\geq& \br{4(2\cL(b)+\zeta(b)) + \frac{n\zeta(b)L}{2b\br{2\cL(b) + \zeta(b)}}}\frac{\sqn{x_0 - x_*}}{\epsilon} \; \implies \;
\ec{F(\bar{x}_k) - F(x_*)} \leq \epsilon.
\end{eqnarray*}
%then, $\ec{F(\bar{x}_k) - F(x_*)} \leq \epsilon$.
\end{corollary}

We define the total complexity as the number of gradients computed per iteration ($b$) times the iteration complexity required to reach an $\epsilon$-approximate solution. Thus, multiplying by $b$ the iteration complexity in Corollary~\ref{cor:complexity-b-SAGA}  and plugging in~\eqref{eq:cL-b-nice}, the total complexity for Algorithm \ref{alg:b-SAGA} is upper bounded by
\begin{align}
K_{saga}(b) &\eqdef \bigg( \frac{4\br{3(n-b)L_{\max} + 2n(b-1)L}}{n-1} + \frac{n(n-b)L_{\max}L}{2\br{3(n-b)L_{\max} + 2n(b-1)L}}  \bigg)\frac{\sqn{x_0 - x_*}}{\epsilon}.\label{eq:total-comp-b-SAGA}
\end{align}
%\begin{align}
%K_{saga}(b) &\eqdef \bigg( \frac{4\br{3(n-b)L_{\max} + 2n(b-1)L}}{n-1}+ \frac{n(n-b)L_{\max}L}{2\br{3(n-b)L_{\max} + 2n(b-1)L}}  \bigg)\frac{\sqn{x_0 - x_*}}{\epsilon}.\label{eq:total-comp-b-SAGA}
%\end{align}

Minimizing this upper bound in the minibatch size $b$ gives us an estimate of the optimal empirical minibatch size, which we verify in our experiments.
\begin{proposition}\label{prop:optimal-minibatch-SAGA-main}
Let $b_{saga}^* = \underset{b \in [n]}{\argmin}\, K_{saga}(b)$, where $K_{saga}(b)$ is defined in \eqref{eq:total-comp-b-SAGA}. 
\begin{itemize}
\item If $L_{\max} \leq \frac{nL}{3}$ then
\begin{eqnarray}\label{eq:b_opt_SAGA}
b_{saga}^* = \left\{
      \begin{array}{ll}
          1 & \mbox{if } \bar{b} < 2 \\
          \floor{b_1} & \mbox{if } 2 \leq \bar{b} < n \\
          n & \mbox{if } \bar{b} \geq n,
      \end{array} \right.
\end{eqnarray}
where $$b_1 \, \eqdef \, \frac{n\br{(n-1)L\sqrt{L_{\max}} - 2\sqrt{2nL - 3L_{\max}}(3L_{\max} - 2L)} }{2(2nL - 3L_{\max})^{\frac{3}{2}}}.$$
\item Otherwise, if $L_{\max} > \frac{nL}{3}$ then $b^* = n$.
\end{itemize}
\end{proposition}

\subsection{Optimal minibatch size for $b$-L-SVRG}\label{sec:optimal-b-L-SVRG}
Since the analysis for Algorithm \ref{alg:b-L-SVRG} is similar to that of Algorithm~\ref{alg:b-SAGA}, we defer its details to the appendix and only present the total complexity and the optimal minibatch size. Indeed, as shown in Section~\ref{sec:app-optimal-minibatch-b-L-SVRG}, an upper bound on the total complexity to find an $\epsilon$-approximate solution for Algorithm~\ref{alg:b-L-SVRG} is given by
\begin{eqnarray}\label{eq:total-comp-b-L-SVRG-main}
K_{svrg}(b) &\eqdef& \br{1+2b}\br{\frac{12\br{3(n-b)L_{\max} + 2n(b-1)L}}{b(n-1)}+\frac{nL}{6}}\frac{\sqn{x_0 - x_*}}{\epsilon}.
\end{eqnarray}

\begin{proposition}\label{prop:optimal-minibatch-b-L-SVRG}
Let $b_{svrg}^* = \underset{b \in [n]}{\argmin}\, K_{svrg}(b)$, where $K_{svrg}(b)$ is defined in \eqref{eq:total-comp-b-L-SVRG}. Then,
\begin{eqnarray}\label{eq:optimal-minibatch-b-L-SVRG}
b_{svrg}^* = 6\sqrt{\frac{n\br{L_{\max} - L}}{72\br{nL-L_{\max}} + n(n-1)L}}.
\end{eqnarray}
\end{proposition}

%Note that even for \textit{SEGA} \cite{Hanzely18} (Algorithm \ref{alg:b-SEGA}), which is a randomized coordinate descent method, we can determine the optimal miniblock size which minimizes the total complexity. The resulting optimal miniblock size we recover is actually $d$, which is equivalent to using full gradient descent (see Section \ref{sec:app-optimal-minibatch-b-SEGA}).

\section{Experiments}
Here we test our new formula for optimal minibatch size of SAGA given by~\eqref{eq:b_opt_SAGA}  against the best minibatch size found over a grid search.   We  used logistic regression with no regularization ($\lambda =0$) to emphasize that our results hold for non-strongly convex functions with data sets taken from the LIBSVM collection~\cite{Chang2011}. For each data set, we ran minibatch SAGA with the stepsize given in Corollary~\ref{cor:complexity-b-SAGA} and until a solution with $$F(x_t) -F(x^*) < 10^{-4}(F(x_0) -F(x^*) )$$ was reached.

%\begin{minipage}[b]{0.45\linewidth}
%
%\end{minipage}
\begin{figure}[th!]
    \vskip 0.2in
    \begin{center}
        \begin{subfigure}[b]{0.45\textwidth}
          \includegraphics[width=\textwidth]{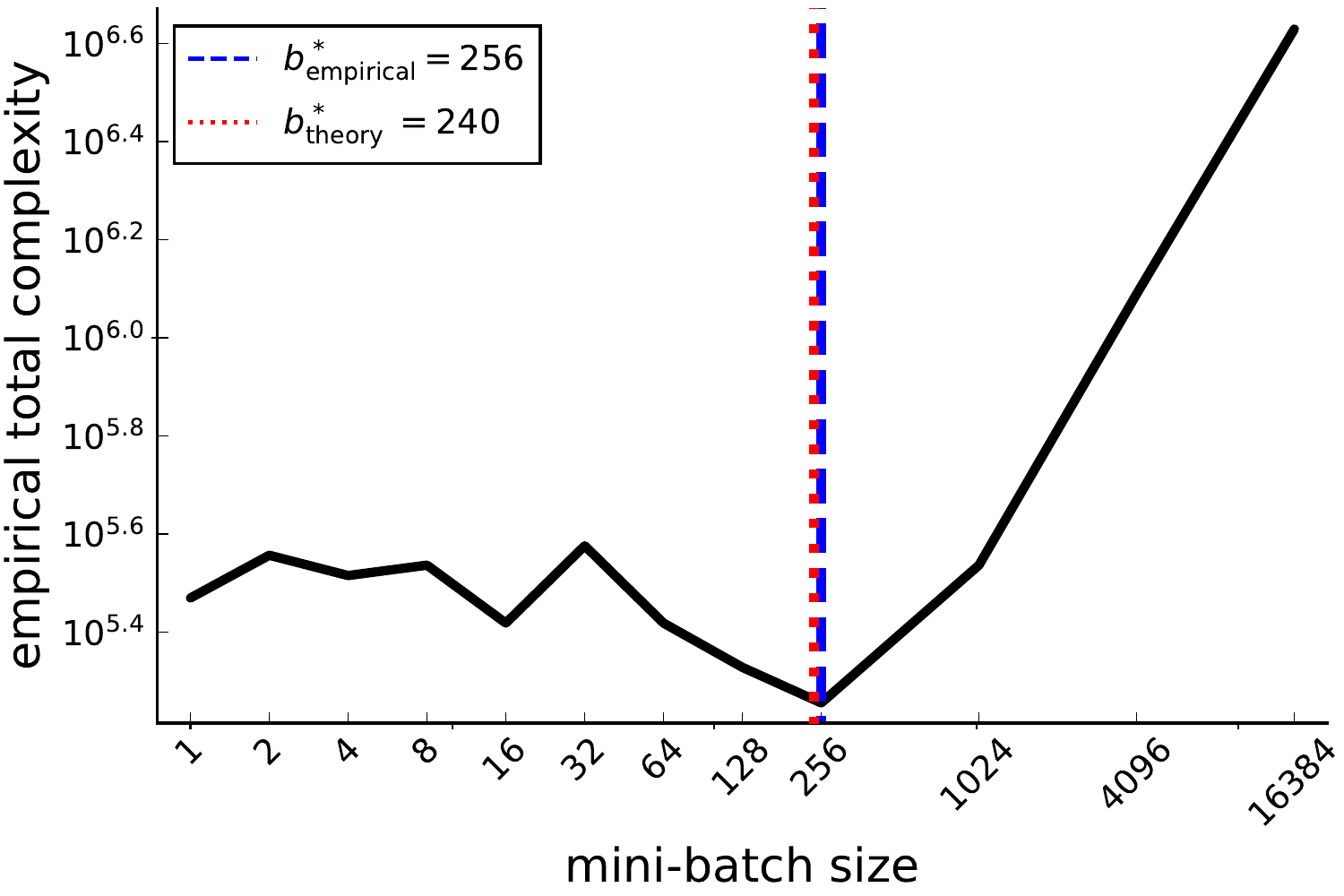}
           \caption{\textit{ijcnn}}
        \end{subfigure}%
        \begin{subfigure}[b]{0.45\textwidth}
          \includegraphics[width=\textwidth]{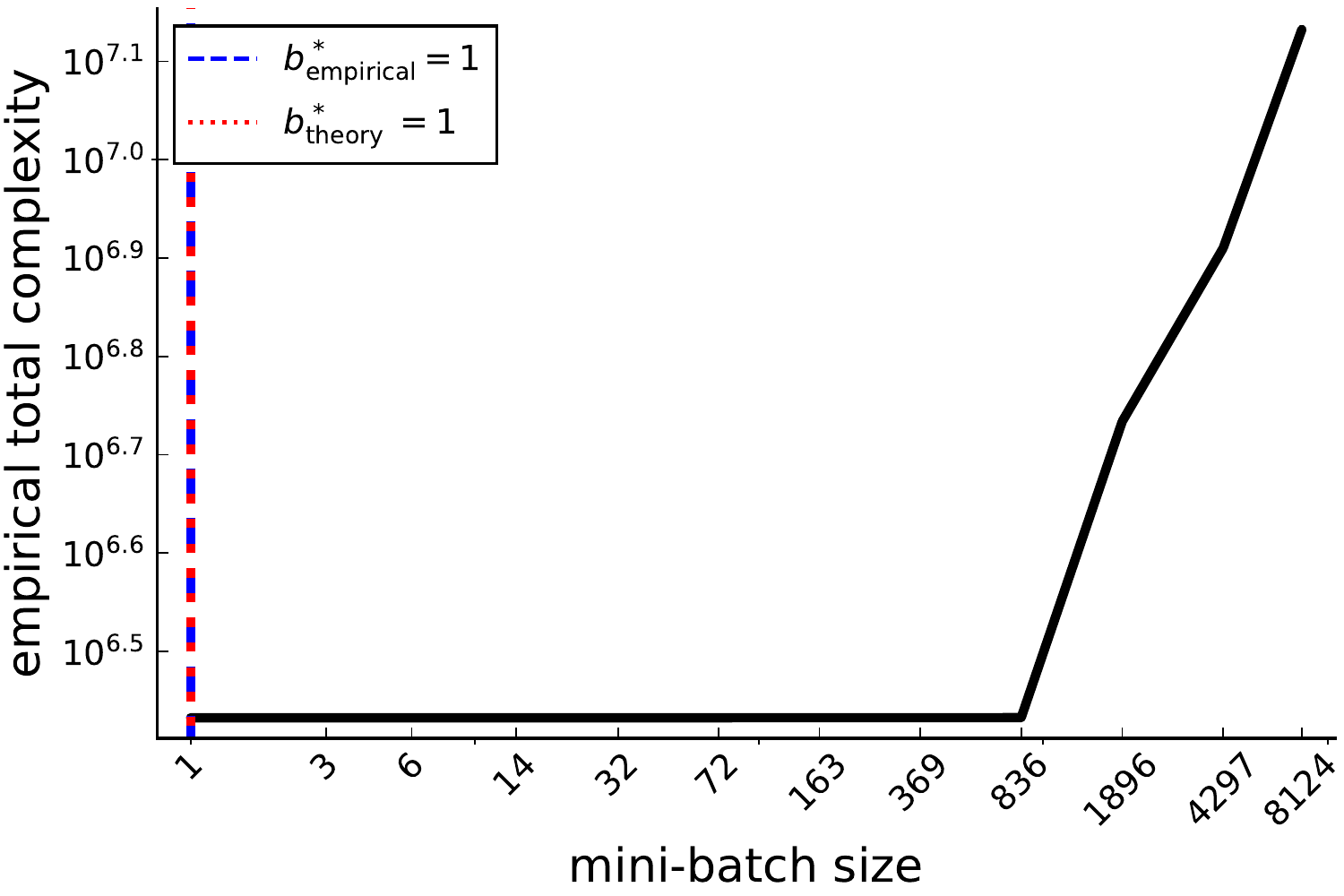}
           \caption{\textit{mushrooms}}
        \end{subfigure}\\
                \begin{subfigure}[b]{0.45\textwidth}
          \includegraphics[width=\textwidth]{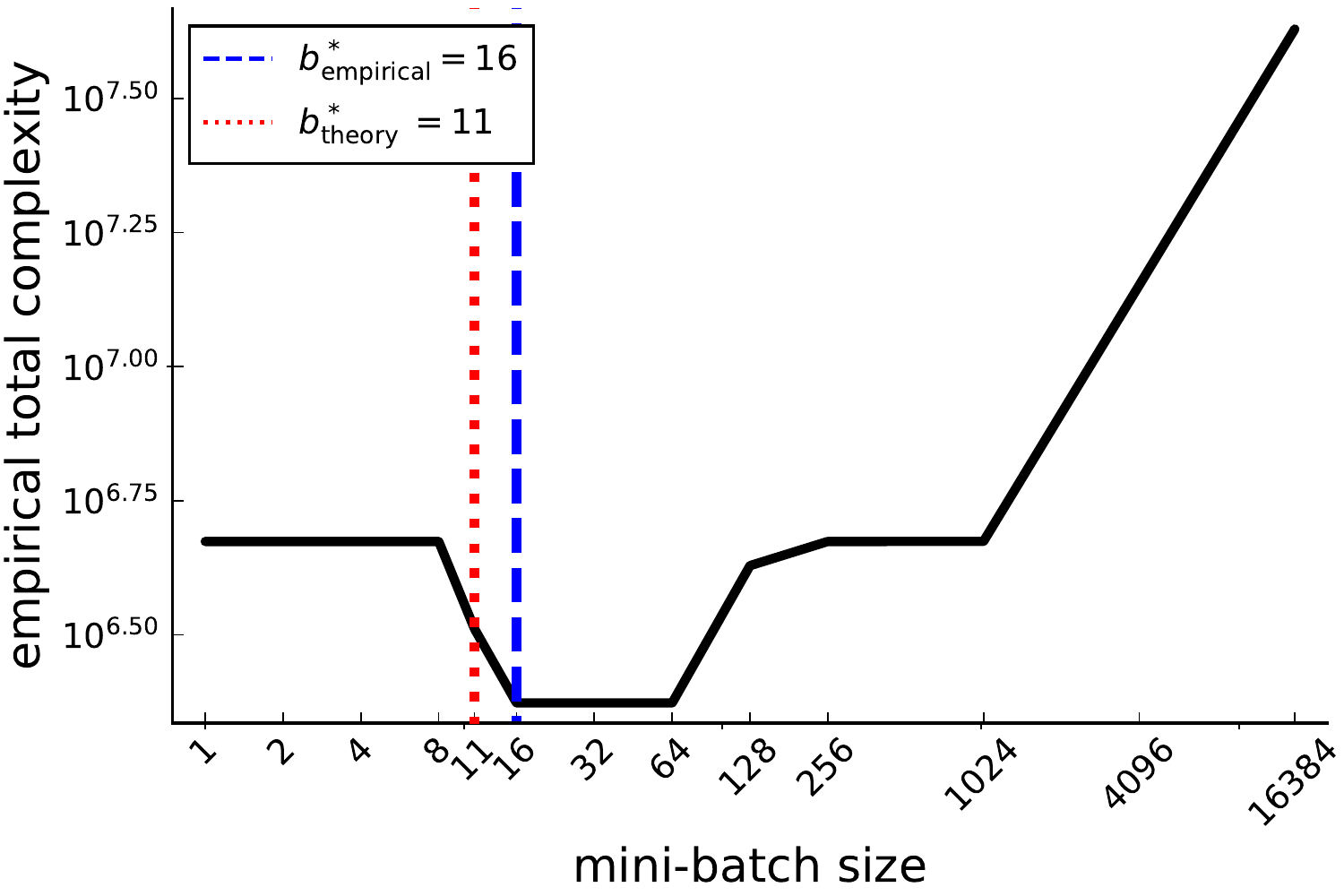}
           \caption{\textit{phishing}}
        \end{subfigure}
                \begin{subfigure}[b]{0.45\textwidth}
           \includegraphics[width=\textwidth]{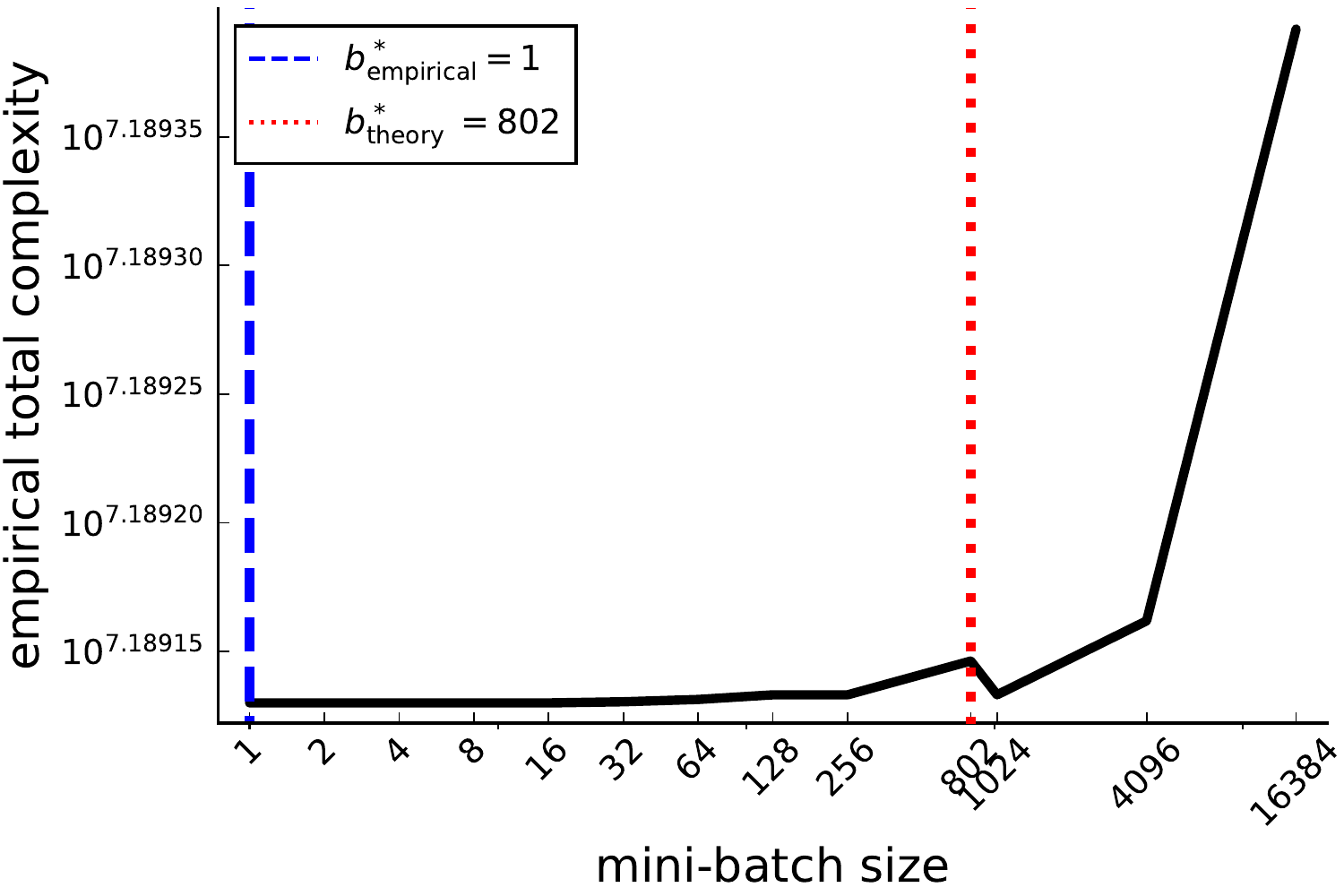}
           \caption{\textit{YearPredictionMSD}}
        \end{subfigure}
                 \caption{Comparing the theoretical optimal batchsize~\eqref{eq:b_opt_SAGA} with the best over a grid.}
%                 $b=1$, $m=n$ for \textit{Free-SVRG} and \textit{L-SVRG-D},
%                Data sets left to right: \textit{ijcnn1, mushrooms, phishing, YearPredictionMSD}.}
                
\label{fig:ijcnn1}
    \end{center}
    \vskip -0.2in
\end{figure}

In   Figure~\ref{fig:ijcnn1} we plot the total complexity (number of iterations times the minibatch size) to reach this tolerance for each minibatch size on the grid. We can see in Figure~\ref{fig:ijcnn1} that for \textit{ijcnn} and \textit{phishing} the optimal minibatch size $b^*_{\text{theory}} =b^*_{saga}$~\eqref{eq:b_opt_SAGA} is remarkably close to the best minibatch size over the grid $b^*_{\text{empirical}}$. Even when $b^*_{\text{theory}}$ is not close to $b^*_{\text{empirical}}$, such as on the \textit{YearPredictionMSD} problem,  the resulting total complexity is still very close to the total complexity of $b^*_{\text{empirical}}$.

%\section*{Broader Impact}
%Our contribution is primarily theoretical. This work does not present any foreseeable ethical or societal consequences.

\section*{Acknowledgements}
Peter Richt\'{a}rik thanks for the support from KAUST through the Baseline Research Fund scheme. Ahmed Khaled and Othmane Sebbouh acknowledge internship support from the Optimization and Machine Learning Lab led by Peter Richt\'{a}rik at KAUST.
Nicolas Loizou acknowledges support by the IVADO Postdoctoral Funding Program.
\bibliographystyle{plain}
\bibliography{../convex_sgd}

\clearpage

\begin{appendices}
\appendixpage

\tableofcontents

\paragraph{Outline of the appendix.} The appendix is organized as follows:
\begin{itemize}
\item \textbf{Section \ref{sec:app-arbitrary-sampling}}: we present the arbitrary sampling framework for Stochastic Gradient methods introduced in \cite{Gower18}, which will be used for the analysis of \textit{SGD} and \textit{L-SVRG}.
\item \textbf{Section \ref{sec:app-corollaries-theorem}}: we present specializations of Theorem \ref{thm:main-prox-dec} to the algorithms we discuss: \textit{SGD}, \textit{DIANA}, \textit{L-SVRG}, \textit{SAGA} and \textit{SEGA}.
\item \textbf{Section \ref{sec:app-unified-analysis-proofs}}: we present the proof of our main Theorem \ref{thm:main-prox-dec}.
\item \textbf{Section \ref{sec:app-main-corollaries-proofs}}: we present the proof of Corollary \ref{cor:conv-dec}.
\item \textbf{Section \ref{sec:app-optimal-minibatch-proofs}}, we present the proof of Proposition \ref{prop:conv-vr-smooth}, and the detailed analysis of the optimal minibatch results for \textit{$b$-SAGA} and \textit{$b$-L-SVRG}, in addition to an analysis for the optimal miniblock size for \textit{$b$-SEGA}.
\item \textbf{Section \ref{sec:auxiliary-lemmas}}: we present some technical lemmas which we use in our analysis.
\end{itemize}

\section{Arbitrary Sampling}\label{sec:app-arbitrary-sampling}
In this section, we recall the arbitrary sampling framework~\cite{Gower19} which allows us to analyze our algorithms for minibatching, importance sampling and virtually all possible forms of sampling.
\subsection{Stochastic reformulation}
To see importance sampling and minibatch variants of stochastic gradient methods all through the same lens, we introduce a \emph{sampling vector} which we will use to re-write~\eqref{eq:optimization_problem}.
\begin{definition}
    We say that a random element-wise positive vector $v \in \R^n_+$ drawn from some distribution $\D$ is a sampling vector if its expectation is the vector of all ones:
    \begin{equation}
        \label{eq:sampling-vector-def}
        \ec[\D]{v_i} = 1, \text { for all } i \in [n].
    \end{equation}
\end{definition}
For a given distribution $\D$ we introduce a \emph{stochastic reformulation} of \eqref{eq:optimization_problem} as follows
\begin{equation}
    \label{eq:stochastic-reformulation}
    \min_{x \in \R^d} \ \pbr{\ec[\D]{f_v (x) \eqdef \frac{1}{n} \sum_{i=1}^{n} v_i f_i (x)} + R(x)}.
\end{equation}
By definition of the sampling vector, $f_v (x)$ and $\nabla f_v (x)$ are unbiased estimators of $f(x)$ and $\nabla f(x)$, respectively, and hence problem \eqref{eq:stochastic-reformulation} is indeed equivalent (i.e.\ a reformulation) of the original problem \eqref{eq:optimization_problem}. In the case of the gradient, for instance, we get
\begin{equation}
    \label{eq:unbiased-sr-gradient}
    \ec[\D]{\nabla f_v (x)} \overset{\eqref{eq:stochastic-reformulation}}{=} \frac{1}{n} \sum_{i=1}^{n} \ec[\D]{v_i} \nabla f_i (x) \overset{\eqref{eq:sampling-vector-def}}{=} \nabla f(x).
\end{equation}
Reformulation \eqref{eq:stochastic-reformulation} can be solved using proximal stochastic gradient descent via
\begin{equation}
    \label{eq:sgd-sr-iterate}
    x_{k+1} = \prox_{\gamma_k R} \br{x_k - \gamma \nabla f_{v_k} (x_k)},
\end{equation}
where $v_k \sim \D$ is sampled i.i.d. at each iteration and $\gamma > 0$ is a stepsize. By substituting specific choices of $\D$, we obtain specific variants of SGD for solving \eqref{eq:optimization_problem}. We further show that \eqref{eq:sgd-sr-iterate} is a special case of \eqref{eq:sgd-proximal-iterates} with a sequence of vectors $g_k = \nabla f_{v_k} (x_k)$ and use the unified analysis in Theorem~\ref{thm:main-prox-dec} to obtain convergence rates for \eqref{eq:sgd-sr-iterate}.

\subsection{Expected Smoothness and Gradient Noise}
In order to analyze \eqref{eq:sgd-sr-iterate} we will make use of the following result, which characterizes the smoothness of the subsampled functions $f_v$.

\begin{lemma}{(Expected Smoothness})
    \label{lem:expected-smoothness}
    If for all $i \in [n], f_i$ is convex and $L_i-$smooth, then there exists a constant $\cL \geq 0$ such that
    \begin{equation}
        \label{eq:expected-smoothness}
        \ec[\D]{\norm{\nabla f_v (x) - \nabla f_v (x_\ast)}^2} \leq 2 \L \, D_f(x,x_\ast),
    \end{equation}
    for all $x \in \R^d$ and where $x_\ast$ is any minimizer of \eqref{eq:optimization_problem}.
\end{lemma}
The proof of this results follows closely that of Lemma 1 in \cite{Gazagnadou19}.
\begin{proof}
Since for all $i \in [n]$, $f_i$ is $L_i$-smooth and convex, we have that each realization $f_v$ (defined in \eqref{eq:stochastic-reformulation}) is $L_v$-smooth and convex. Thus, from Lemma \ref{lem:smoothconvexaroundxst}, we have that for all $x \in \R^d$,
\begin{eqnarray*}
\sqn{\nabla f_v(x) - f_v(x_*)} &\leq& 2L_v \br{f_v(x) - f_v(x_*) - \dotprod{\nabla f_v(x_*), x - x_*}}\\
&=& \frac{2}{n}\sum_{i=1}^n L_vv_i \br{f_i(x) - f_i(x_*) - \dotprod{\nabla f_i(x_*), x - x_*}}.
\end{eqnarray*}
Taking expectation over the samplings,
\begin{eqnarray*}
\ecd{\sqn{\nabla f_v(x) - f_v(x_*)}} &\leq& \frac{2}{n}\sum_{i=1}^n \ecd{v_iL_v}\br{f_i(x) - f_i(x_*) - \dotprod{\nabla f_i(x_*), x - x_*}}\\
&\leq& 2\max_{j=1,\dots,n}\ecd{L_vv_j} \br{f(x) - f(x_*) - \dotprod{\nabla f(x_*), x - x_*}}\\
&=& 2 \max_{j=1,\dots,n}\ecd{L_vv_j} D_f(x, x_*).
\end{eqnarray*}
\end{proof}

%\begin{lemma}
%\label{lem:bniceconst} %Let $v$ be a sampling vector with $v_i \geq 0$ with probability one .
% If $f_\xi$ is $L$--smooth  and there exists $x^*\in\cX^*$ such that $f_\xi$ convex around $x^*$ that is
% \begin{align} \label{eq:smoothnessfuncmain}
% f_\xi(y) -f_\xi(x) &\leq  \dotprod{\nabla f_\xi(x), y-x} +\frac{L}{2}\norm{y-x}_2^2, \quad  \forall x,y\in\R^d,\\
%f_\xi(x^*) -f_\xi(x) & \leq  \dotprod{\nabla f_\xi(x^*), x^*-x}, \quad  \forall x\in\R^d, \forall \xi \sim \D,\label{eq:convstarmain}
%\end{align}
% then  Assumption~\ref{asm:expected-smoothness} holds with
%\end{lemma}

Next, we define the gradient noise.
\begin{definition}{(Gradient Noise).}\label{def:gradient-noise}
    The gradient noise $\sigma^2 = \sigma^2(f, \D)$ is defined by
    \begin{equation}
        \label{eq:gradient-noise-def}
        \sigma^2 \eqdef \ec[\D]{\norm{\nabla f_v (x_\ast)-\nabla f(x_\ast)  }^2}.
    \end{equation}
    \end{definition}

\subsection{Minibatching elements without replacement}\label{sec:AS-minibatching}
Since analyzing minibatching for variance reduced methods is one of the main focuses of our work, we present minibatching without replacement as an example of the use of arbitrary sampling.

First, we define samplings.

\begin{definition}[Sampling]\label{def:sampling}
A sampling  $S\subseteq [n]$ is any random set-valued map which is uniquely defined by the probabilities $\sum_{B \subseteq [n]} p_B =1$ where
$p_{B} \;\eqdef \; \mathbb{P}(S=B), \quad \forall B \subseteq [n].$
A sampling $S$ is called proper if for every $i \in [n]$, we have that  $p_i \eqdef \Proba(i \in S) =  \underset{C:i\in C}{\sum}p_C > 0$.
\end{definition}

We can build a sampling vector using a sampling as follows.

\begin{lemma}[Sampling vector, Lemma 3.3 in \cite{Gower19}] \label{lem:sampling_vector}
Let $S$ be a proper sampling. Let $p_i \eqdef \Proba(i \in S)$ and $\mathbf{P} \eqdef \diag\left(p_1,\dots,p_n\right)$. Let $v = v(S)$ be a random vector defined by
\begin{eqnarray} \label{eq:vSdef}
v(S) \;= \; \mathbf{P}^{-1}\sum_{i \in S}e_i  \;\eqdef \; \mathbf{P}^{-1}e_S.
\end{eqnarray}
It follows that $v$ is a sampling vector.
\end{lemma}
\emph{Proof.}
The $i$-th coordinate of $v(S)$ is $v_i(S) =  \mathbbm{1}(i \in S) / p_i$ and thus
\[\ec{v_i(S)}\; =\; \frac{\ec{\mathbbm{1}(i \in S)}}{ p_i} \;=\; \frac{\Proba(i \in S)}{p_i} \;= \;1.  \qquad \qquad \qed\]
%The sampling vectors always verify the unbiasedness condition in Definition~\ref{ass:unbsn}  by construction. Indeed,

Next, we define $b$-nice sampling, also known as minibatching without replacement.

\begin{definition}[$b$-nice sampling] \label{def:bnice_sampling}
$S$ is a $b$-nice sampling if it is a sampling such that
% a probability distribution given by
  \[\Proba(S = B) = \frac{1}{\binom{n}{b}}, \quad \forall	B \subseteq [n],\; \text{ with } \;|B| = b.\]
\end{definition}

To construct such a sampling vector based on the $b$--nice sampling, note that $p_i = \tfrac{b}{n}$ for all $i \in [n]$ and thus we have that $v(S) = \tfrac{n}{b}\sum_{i\in S}e_i$ according to Lemma \ref{lem:sampling_vector}. The resulting subsampled function is then $f_v(x) = \tfrac{1}{|S|}\sum_{i\in S}f_i(x)$, which is simply the minibatch average over $S$.

A remarkable result for $b$-nice sampling is that when all the functions $f_i, i\in [n]$ are $L_i-$smooth and convex, then the expected smoothness constant \eqref{eq:expected-smoothness} nicely interpolates between $L$, the smoothness constant of $f$, and $L_{\max} = \underset{i\in [n]}{\max}\,L_i$.

\begin{lemma}[$\cL$ for $b-$nice sampling, Proposition 3.8 in \cite{Gower19}]
Let $v$ be a sampling vector based on the $b-$nice sampling defined in \ref{def:bnice_sampling}. If for all $i \in [n], f_i$ is convex and $L_i-$smooth, then \eqref{eq:expected-smoothness} holds with
    \begin{equation*}
        \cL(b) = \frac{1}{b}\frac{n-b}{n-1}L_{\max} + \frac{n}{b}\frac{b-1}{n-1}L,
    \end{equation*}
    where $L$ is the smoothness constant of $f$ and $L_{\max} = \underset{i\in [n]}{\max} \, L_i$.
\end{lemma}

\section{Notable Corollaries of Theorem \ref{thm:main-prox-dec}}\label{sec:app-corollaries-theorem}
In this section, we present corollaries of Theorem \ref{thm:main-prox-dec} for five algorithms:  
\begin{itemize}
\item \textit{SGD} with arbitrary sampling (Algorithm \ref{alg:SGD-AS}).
\item \textit{DIANA} (Algorithm \ref{alg:DIANA}).
\item \textit{L-SVRG} with arbitrary sampling (Algorithm \ref{alg:L-SVRG-AS}), and minibatch \textit{L-SVRG} as a special case (Algorithm \ref{alg:b-L-SVRG}).
\item Minibatch \textit{SAGA} (Algorithm \ref{alg:b-SAGA}).
\item Miniblock \textit{SEGA} (Algorithm \ref{alg:b-SEGA}).
\end{itemize}
This means that for each method, we will present the constants which satisfy Assumption \ref{asm:main_assumption} and specialize Theorem \ref{thm:main-prox-dec} using these constants.

\subsection{SGD with arbitrary sampling}
\begin{algorithm}[H]
  \begin{algorithmic}%[1]
    \State \textbf{Parameters} step sizes $(\gamma_k)_k$, a sampling vector $v \sim \D$
    \State \textbf{Initialization}   $x_0 \in \mathbb{R}^d$
    % \For {$k=0, 1, 2,\dots$}\vskip 1ex
    \For {$k=1, 2,\dots$}\vskip 1ex
      \State Sample $v_k \sim \D$
      \State $g_k = \nabla f_{v_k}(x_k)$
      \State $x_{k+1} = \prox_{\gamma_k R}\br{x_k - \gamma_k g_k}$
    \EndFor

  \end{algorithmic}
  \caption{SGD-AS}
  \label{alg:SGD-AS}
\end{algorithm}

\begin{lemma}\label{lem:constants-SGD-AS}
The iterates of Algorithm \ref{alg:SGD-AS} satisfy Assumption \ref{asm:main_assumption} with
\begin{eqnarray*}
\sigma_k^2 = 0
\end{eqnarray*}
and constants:
\begin{eqnarray}\label{eq:params-SGD-AS}
A = 2\cL, \; B = 0, \; \rho = 1, \; C = 0, \; D_1 = 2\sigma^2, \; D_2 = 0,
\end{eqnarray}
where $\cL$ is defined in \eqref{eq:expected-smoothness} and $\sigma^2$ in \eqref{eq:gradient-noise-def}.
\end{lemma}
\begin{proof}
See Lemma A.2 in \cite{Gorbunov19}.
\end{proof}

Using the constants given in the above lemma, we have the following immediate corollary of Theorem~\ref{thm:main-prox-dec}.
\begin{corollary}
    \label{cor:conv-SGD-AS}
    Assume that $f$ has a finite-sum structure~\eqref{eq:finite-sum} and that Assumption \ref{asm:function-class} holds.  Let $(\gamma_k)_{k\geq 0}$ be a decreasing, strictly positive sequence of step sizes chosen such that
    \[ 0 < \gamma_0 < \min \left\{ \frac{1}{4\cL}, \frac{1}{L} \right\}. \]
    Then, from Theorem \ref{thm:main-prox-dec} and Lemma \ref{lem:constants-SGD-AS}, we have that the iterates given by Algorithm \ref{alg:SGD-AS} verify
    \begin{align}
        \label{eq:conv-SGD-AS}
        \ec{F(\bar{x}_t) - F(x_\ast)} \leq \frac{\sqn{x_0 - x_*} + 2\gamma_0\br{F(x_0) - F(x_*)} + 4\sigma^2\sum_{k=0}^{t-1}\gamma_k^2}{2\sum_{i=0}^{t-1}\br{1 - 4\gamma_i\cL} \gamma_i},
    \end{align}
    where $\bar{x}_t \eqdef \sum\limits_{k=0}^{t-1} \frac{\br{ 1 - 4\gamma_k\cL}  \gamma_k}{\sum_{i=0}^{t-1}\br{1 - 4\gamma_i\cL} \gamma_i}x_k$.
\end{corollary}

\subsection{DIANA}\label{sec:DIANA-app}
A complete description of the \textit{DIANA} algorithm can be found in \cite{Mishchenko19}. 

To analyze the \textit{DIANA} algorithm (Algorithm \ref{alg:DIANA}), we introduce quantization operators.
\begin{definition}[$w$-quantization operator, Definition 4 in \cite{Mishchenko19}]
Let $w > 0$. A random operator $Q: \R^d \rightarrow \R$ with the properties:
\begin{eqnarray}
\ec{Q(x)} = x, \quad \ec{\sqn{Q(x)}} \leq (1+w)\sqn{x},
\end{eqnarray}
for all $x \in \R^d$ is called a $w$-quantization operator.
\end{definition}
Several examples of quantization operators can be found in \cite{Mishchenko19}.
 
\begin{algorithm}[h]
  \begin{algorithmic}%[1]
    \State \textbf{Parameters} $w$-quantization operator $Q$, Learning rates $\alpha > 0$ and $\gamma >0$, initial vectors $x^0, h_1^0,\dots,h_n^0 \in \R^d$ and $h^0 = \frac{1}{n}\sum\limits_{i=1}^n h_i^0$
    % \For {$k=0, 1, 2,\dots$}\vskip 1ex
    \State \textbf{Initialization} $x^0, h_1^0,\dots,h_n^0 \in \R^d$
    \State Set $h^0 = \frac{1}{n}\sum\limits_{i=1}^n h_i^0$
    \For {$k=1, 2,\dots$}\vskip 1ex
    		\State  Broadcast $x_k$ to all workers.
    		\For {$k=1, 2,\dots$}\vskip 1ex
    			\State  Sample $g_i^k$ such that $\ec[k]{g_i^k} = \nabla f_i(x_k)$
    			\State $\Delta_i^k = g_i^k - h_i^k$
    			\State Sample $\hat{\Delta}_i^k \sim Q(\Delta_i^k)$
    			\State $h_i^{k+1} = h_i^k + \alpha \Delta_i^k$
    			\State $\hat{g}_i^k = h_i^k + \hat{\Delta}_i^k$ 
    		\EndFor
      	\State $\hat{\Delta}^k = \frac{1}{n}\sum\limits_{i=1}^n\Delta_i^k$
      	\State $g_k = \frac{1}{n}\sum\limits_{i=1}^n\hat{g}_i^k = h^k + \hat{\Delta}^k$
      	\State $x_{k+1} = \prox_{\gamma_k R}\br{x_k - \gamma_k g_k}$
      	\State $h^{k+1} = \frac{1}{n}\sum\limits_{i=1}^n h_i^{k+1} = h^k + \alpha \hat{\Delta}^k$
    \EndFor
  \end{algorithmic}
  \caption{DIANA}
  \label{alg:DIANA}
\end{algorithm}
%
%For Algorithm \ref{alg:DIANA} let $\sigma^2 > 0$ be such that for all $k \in \N$:
%\begin{eqnarray}
%\frac{1}{n}\sum_{i=1}^n\ec[k]{\sqn{g_i^k - \nabla f(x_k)}} \leq \sigma^2. 
%\end{eqnarray}

For convenience, we repeat the statement of Lemma~\ref{lem:constants-DIANA-main} below.
\begin{lemma}\label{lem:constants-DIANA}
Assume that $f$ has a finite sum structure and that Assumption~\ref{asm:function-class} holds. The iterates of \textit{DIANA} (Algorithm \ref{alg:DIANA}) satisfy Assumption \ref{asm:main_assumption} with constants:
\begin{eqnarray}
A = \br{1+\frac{2w}{n}}L_{\max}, \; B = \frac{2w}{n}, \; \rho = \alpha, \; C = L_{\max}\alpha, \; D_1 = \frac{(1+w)\sigma^2}{n}, \;  D_2 = \alpha\sigma^2,
\end{eqnarray}
where $w > 0$ and $\alpha \leq \frac{1}{1+w}$ are parameters of Algorithm \ref{alg:DIANA} and $\sigma^2$ is such that $$\forall k \in \N, \quad \frac{1}{n}\sum_{i=1}^n\ec{\sqn{g_i^k - \nabla f(x_k)}} \leq \sigma^2.$$
%
%
%The iterates of Algorithm \ref{alg:DIANA} satisfy Assumption \ref{asm:main_assumption} with
%\begin{eqnarray*}
%\sigma_k^2 \eqdef \frac{1}{n}\sum_{i=1}^n \sqn{h_i^k - \nabla f(x_*)}
%\end{eqnarray*}
%and constants:
%\begin{eqnarray}\label{eq:params-DIANA}
%A = \br{1+\frac{2w}{n}}L_{\max}, \; B = \frac{2w}{n}, \; \rho = \alpha, \; C = L_{\max}\alpha, \; D_1 = \frac{(1+w)\sigma^2}{n}, \;  D_2 = \alpha\sigma^2.
%\end{eqnarray}
\end{lemma}
\begin{proof}
See Lemma A.12 in \cite{Gorbunov19}.
\end{proof}
Now using the constants given in the above lemma in Theorem~\ref{thm:main-prox-dec} gives the following corollary.
\begin{corollary}
    \label{cor:conv-DIANA}
    
    Assume that $f$ has a finite sum structure \eqref{eq:finite-sum} and that Assumption \ref{asm:function-class} holds. Let $(\gamma_k)_{k\geq 0}$ be a decreasing, strictly positive sequence of step sizes chosen such that
    \[ 0 < \gamma_0 < \frac{1}{2 (1 + \frac{4w}{n})L_{\max}}. \]
     Then, from Theorem \ref{thm:main-prox-dec} and Lemma \ref{lem:constants-DIANA}, we have that the iterates given by Algorithm \ref{alg:DIANA} verify
    \begin{align}
        \label{eq:conv-DIANA}
        \ec{F(\bar{x}_t) - F(x_\ast)} \leq \frac{\sqn{x_0 - x_*} + 2\gamma_0\br{F(x_0) - F(x_*)+ \frac{2w\gamma_0}{\alpha n}\sigma_0^2 } +  \frac{2\br{1+5w}\sigma^2}{n}  \sum_{k=0}^{t-1}\gamma_k^2}{2\sum_{i=0}^{t-1}\br{1 - \gamma_i\eta} \gamma_i},
    \end{align}
    where $\eta \eqdef 2(1+\frac{4w}{n})L_{\max}$,  $\bar{x}_t \eqdef \sum\limits_{k=0}^{t-1} \frac{\br{ 1 - \gamma_k\eta}  \gamma_k}{\sum_{i=0}^{t-1}\br{1 - \gamma_i\eta} \gamma_i}x_k$ and ${\delta_0 \eqdef F(x_0) - F(x_\ast)}$.
\end{corollary}

\subsection{L-SVRG with arbitrary sampling}
\begin{algorithm}[H]
  \begin{algorithmic}%[1]
    \State \textbf{Parameters} step size $\gamma$, sampling vector $v \sim \D$
    \State \textbf{Initialization}   $w_0 = x_0 \in \mathbb{R}^d$
    % \For {$k=0, 1, 2,\dots$}\vskip 1ex
    \For {$k=1, 2,\dots$}\vskip 1ex
      \State Sample $v_k \sim \D$
      \State $g_k = \nabla f_{v_k}(x_k) - \nabla f_{v_k}(w_k) + \nabla f(w_k)$ % 	 	\Comment Stochastic Gradient estimate. \label{com:sto_step}
      \State $x_{k+1} = \prox_{\gamma R}\br{x_k - \gamma g_k}$
      \State $ w_{k+1} = \left\{
      \begin{array}{ll}
          x_k & \mbox{with probability }p \\
          w_{k} & \mbox{with probability } 1-p
      \end{array} \right.$
    \EndFor
  \end{algorithmic}
  \caption{L-SVRG-AS}
  \label{alg:L-SVRG-AS}
\end{algorithm}

\begin{lemma}\label{lem:constants-L-SVRG-AS-source}
If Assumption~\ref{asm:function-class} holds then the iterates of Algorithm \ref{alg:L-SVRG-AS} satisfy
\begin{eqnarray}
\ec[k]{\sqn{g_k - \nabla f(x_*)}} &\leq& 4\cL D_f(x_k, x_*) + 2\sigma_k^2 \\
\ec[k]{\sigma_{k+1}^2} &\leq& (1-p)\sigma_k^2 + 2p\cL D_f(x_k, x_*),
\end{eqnarray}
where
\begin{eqnarray}\label{eq:sigma-L-SVRG-AS}
\sigma_k^2 = \ecd{\sqn{\nabla f_{v_k}(w_k) - \nabla f_{v_{k}}(x_*) - \br{\nabla f(w_k) - \nabla f(x_*)}}}
\end{eqnarray}
and $\cL$ is defined in \eqref{eq:expected-smoothness}.
\end{lemma}

\begin{proof}
By Lemma~\ref{lem:expected-smoothness} we have that~\eqref{eq:expected-smoothness} holds with $\cL>0.$ Furthermore
\begin{eqnarray*}
\ec[k]{\sqn{g_k}} &=& \ec[k]{\sqn{\nabla f_{v_k}(x_k) - \nabla f_{v_k}(w_k) + \nabla f(w_k) - \nabla f(x_*)}} \\
&\leq& 2 \ec[k]{\sqn{\nabla f_{v_k}(x_k) - \nabla f_{v_k}(x_*)}}\\
&&+ 2 \ec[k]{\sqn{\nabla f_{v_k}(w_k) - \nabla f_{v_{k}}(x_*) - \br{ \nabla f(w_k) - \nabla f(x_*)} }},
\end{eqnarray*}
where we used in the inequality that for all $a, b \in \R^d, \sqn{a + b} \leq 2\sqn{a} + 2\sqn{b}$. Thus,
\begin{eqnarray*}
\ec[k]{\sqn{g_k}} \overset{\eqref{eq:expected-smoothness}}{\leq}  4\cL D_f\br{x_k, x_*} + 2 \sigma_k^2.
\end{eqnarray*}
Moreover,
\begin{eqnarray*}
\ec[k]{\sigma_{k+1}} &=& (1-p)\sigma_k^2 + p\ec[k]{\sqn{\nabla f_{v_k}(x_k) - \nabla f_{v_{k}}(x_*) - \br{\nabla f(x_k) - \nabla f(x_*)}}}\\
&\overset{\eqref{eq:expected-smoothness}}{\leq}& (1-p)\sigma_k^2 + 2p\cL D_f\br{x_k, x_*},
\end{eqnarray*}
where we also used in the last inequality that $\ecn{ X - \ec{X} } = \ecn{X} - \norm{\ec{X}}^2 \leq \ecn{X}$.
\end{proof}

We have the following immediate consequence of the previous lemma.
\begin{lemma}\label{lem:constants-L-SVRG-AS}
If Assumption~\ref{asm:function-class} holds then the iterates of Algorithm \ref{alg:L-SVRG-AS} satisfy Assumption \ref{asm:main_assumption} with
\begin{eqnarray*}
\sigma_k^2 = \ecd{\sqn{\nabla f_v(x_k) - \nabla f_v(w_k) + \nabla f(w_k)}}
\end{eqnarray*}
and constants
\begin{eqnarray}\label{eq:params-L-SVRG-AS}
A = 2\cL, \; B = 2, \; \rho = p, \; C = p\cL, \; D_1 = D_2 = 0,
\end{eqnarray}
where $\cL$ is defined in \eqref{eq:expected-smoothness}.
\end{lemma}
Using the constant derived in Lemma \ref{lem:constants-L-SVRG-AS} in Theorem \ref{thm:main-prox-dec} gives the following corollary.
\begin{corollary}
    \label{cor:conv-L-SVRG-AS}
    Assume that $f$ has a finite sum structure \eqref{eq:finite-sum} and that Assumption \ref{asm:function-class} holds. Let $\gamma_k = \gamma$ for all $k \in \N$, where
    \[ 0 < \gamma < \min \left\{ \frac{1}{8\cL}, \frac{1}{L} \right\}. \]
     Then, from Theorem \ref{thm:main-prox-dec} and Lemma \ref{lem:constants-L-SVRG-AS}, we have that the iterates given by Algorithm \ref{alg:L-SVRG-AS} verify
    \begin{align}
        \label{eq:conv-L-SVRG-AS}
        \ec{F(\bar{x}_t) - F(x_\ast)} \leq \frac{\sqn{x_0 - x_*} + 2\gamma\br{F(x_0) - F(x_*) + \frac{2\gamma}{p}\sigma_0^2}}{2\gamma\br{1 - 8\gamma\cL} t},
    \end{align}
    where $\bar{x}_t \eqdef \frac{1}{t}\sum\limits_{k=0}^{t-1} x_k$ and where $\cL$ is defined in \eqref{eq:expected-smoothness}.
\end{corollary}

\subsubsection{$b$-L-SVRG}
As we demonstrated in Section \ref{sec:AS-minibatching}, we can specialize the results derived for arbitrary sampling to minibatching without replacement by using a $b-$nice sampling defined in Definition \ref{def:bnice_sampling} and the corresponding sampling vector \eqref{eq:vSdef}.

Indeed, using Algorithm \ref{alg:L-SVRG-AS} with $b$-nice sampling is equivalent to using Algorithm \ref{alg:b-L-SVRG}. Thus, we have the following lemma.

\begin{corollary}\label{cor:constants-b-L-SVRG}
From Lemma \ref{lem:constants-L-SVRG-AS}, we have that the iterates of Algorithm \ref{alg:b-L-SVRG} satisfy Assumption \ref{asm:main_assumption} with constants:
\begin{eqnarray}
A = 2\cL(b), \; B = 2, \; \rho = p, \; C = p\cL(b), \; D_1 = D_2 = 0,
\end{eqnarray}
where $\cL(b)$ is defined in \eqref{eq:cL-b-nice}.
\end{corollary}
A convergence result for Algorithm \ref{alg:b-L-SVRG} can be easily concluded from Corollary~\ref{cor:conv-L-SVRG-AS}, with $\cL(b)$ in place of $\cL$.

\subsection{$b$-SAGA}
Lemma \ref{lem:constants-b-SAGA-main} in the main text is a consequence of the following lemma.
\begin{lemma}\label{lem:constants-b-SAGA}
Consider the iterates of Algorithm \ref{alg:b-SAGA}. We have:
\begin{eqnarray}
\ec[k]{\sqn{g_k}} &\leq& 4\cL(b)\br{f(x_k) - f(x_*)} + 2\sigma_k^2 \label{eq:bnd_gk_b-SAGA} \\
\ec[k]{\sigma_{k+1}^2} &\leq& (1-\frac{b}{n})\sigma_k^2 + 2\frac{b\zeta(b)}{n} \br{f(x_k) - f(x_*)}, \label{eq:bnd_sigmak_b-SAGA}
\end{eqnarray}
where:
\begin{eqnarray}
\sigma_k^2 = \frac{1}{nb} \frac{n-b}{n-1}\trn{J_k - \nabla H(x_*)} \quad \mbox{and} \quad \zeta(b) \eqdef \frac{1}{b}\frac{n-b}{n-1}L_{\max},
\end{eqnarray}
with $\trn{Z} = \tr(Z^\top Z)$ for any $Z \in \R^{d \times n}$.
\end{lemma}
\begin{proof} The inequality
\eqref{eq:bnd_gk_b-SAGA} corresponds to Lemma 3.10 and \eqref{eq:bnd_sigmak_b-SAGA} to Lemma 3.9 in \cite{Gower18}.
\end{proof}

The previous Lemma gives us the constants for Assumption \ref{asm:main_assumption} for Algorithm \ref{alg:b-SAGA}.
\begin{lemma}\label{lem:constants-b-SAGA-app}
The iterates of Algorithm \ref{alg:b-SAGA} satisfy Assumption \ref{asm:main_assumption} with
\begin{eqnarray}
\sigma_k^2 = \frac{1}{nb} \frac{n-b}{n-1}\trn{J_k - \nabla H(x_*)}
\end{eqnarray}
and constants
\begin{eqnarray}\label{eq:constants-b-SAGA-app}
A = 2\cL(b), \; B = 2, \; \rho = \frac{b}{n}, \; C = \frac{b\zeta(b)}{n}, \; D_1 = D_2 = 0.
\end{eqnarray}
\end{lemma}

Using the constant derived in Lemma \ref{lem:constants-b-SAGA-app} in Theorem \ref{thm:main-prox-dec} gives the following corollary.

\begin{corollary}
    \label{cor:conv-b-SAGA}
    Assume that $f$ has a finite sum structure \eqref{eq:finite-sum} and that Assumption \ref{asm:function-class} holds. Choose for all $k \in \N$ $\gamma_k = \gamma$, where
    \[ 0 < \gamma < \frac{1}{2 (2\cL(b) + \zeta(b))}. \]
     Then, from Theorem \ref{thm:main-prox-dec} and Lemma \ref{lem:constants-b-SAGA-app}, we have that the iterates given by Algorithm \ref{alg:b-SAGA} verify
    \begin{align}
        \label{eq:conv-b-SAGA}
        \ec{F(\bar{x}_t) - F(x_\ast)} \leq \frac{\sqn{x_0 - x_*} + 2\gamma\br{F(x_0) - F(x_*) + \frac{2n\gamma}{b} \sigma_0^2}}{2\gamma\br{1 - 2\gamma\br{2\cL(b)+2\zeta(b)}} t},
    \end{align}
    where $\bar{x}_t \eqdef \frac{1}{t}\sum\limits_{k=0}^{t-1} x_k$.
\end{corollary}

\subsection{$b$-SEGA}
\begin{lemma}\label{lem:constants-b-SEGA-source}
Consider the iterates of Algorithm \ref{alg:b-SEGA}. We have:
\begin{eqnarray}
\ec[k]{\sqn{g_k}} &\leq& \frac{4dL}{b}D_f\br{x_k, x_*} + 2\br{\frac{d}{b} - 1}\sigma_k^2 \\
\ec[k]{\sigma_{k+1}^2} &\leq& (1-\frac{b}{d})\sigma_k^2 + \frac{2bL}{d}D_f\br{x_k, x_*},
\end{eqnarray}
where:
\begin{eqnarray}
\sigma_k^2 = \sqn{h_k - \nabla f(x_*)}.
\end{eqnarray}
\end{lemma}

\begin{proof}
Let $S$ be a random miniblock s.t. $\Proba(S = B) = \frac{1}{\binom{n}{b}}$ for any $B \subseteq [n]$ s.t. $|B| = b$. Then, for any vector $a = [a_1,\dots,a_n] \in \R^d$, we have:
\begin{eqnarray}\label{eq:sqn-sampled-miniblock}
\ec{\sqn{I_Sa}} = \frac{b}{d}\sqn{a} \quad \mbox{and} \quad \ec{\sqn{(I-\frac{d}{b}I_S)a}} = \br{\frac{d}{b} - 1}\sqn{a}. 
\end{eqnarray}  
Indeed,
\begin{eqnarray*}
\ec{\sqn{I_S a}} &=& \ec{\sum_{i \in S} a_i^2} = \sum_{B \subseteq [d], |B| = b}\Proba(S = B)\sum_{i \in B} a_i^2 = \frac{1}{\binom{d}{b}}\sum_{B \subseteq [d], |B| = b}\sum_{i = 1}^d a_i^2 \1_B(i)\\
&=& \frac{1}{\binom{d}{b}}\sum_{i = 1}^d a_i^2 \sum_{B \subseteq [d], |B| = b} \1_B(i) = \frac{\binom{d-1}{b-1}}{\binom{d}{b}}\sum_{i = 1}^d a_i^2 = \frac{b}{d}\sqn{a},
\end{eqnarray*}
where we used that $|B \in [d]: |B|=b \land i \in B| = \binom{d-1}{b-1}$. And
\begin{eqnarray*}
\sqn{(I-\frac{d}{b}I_S)a} &=& \sum_{i\in S}\br{1-\frac{d}{b}}^2 a_i^2 + \sum_{i \notin S} a_i^2 = \frac{d^2 - 2bd}{b^2}\sum_{i \in S}a_i^2 + \sqn{a} \\
&=&  \frac{d^2 - 2bd}{b^2} \sqn{I_{S}a} + \sqn{a}.
\end{eqnarray*}
Thus,
\begin{eqnarray*}
\ecn{(I-\frac{d}{b}I_S)a} = \br{\frac{d^2 - 2bd}{b^2}\frac{b}{d} + 1}\sqn{a} = \br{\frac{d}{b} - 1}\sqn{a}.
\end{eqnarray*}

We have
\begin{eqnarray*}
\ec[k]{\sqn{g_k - \nabla f(x_*)}} &=& \ec[k]{\sqn{\frac{d}{b}I_{B_k}(\nabla f(x_k)- \nabla f(x_*)) + \br{I - \frac{d}{b}I_{B_k}}(h_k- \nabla f(x_*))}} \\
&\leq& \frac{2d^2}{b^2} \ec[k]{\sqn{I_{B_k}(\nabla f(x_k)- \nabla f(x_*))}} + 2 \ec[k]{\sqn{\br{I - \frac{d}{b}I_{B_k}}(h_k- \nabla f(x_*))}} \\
&\overset{\eqref{eq:sqn-sampled-miniblock}}{=}& \frac{2d}{b}\sqn{\nabla f(x_k) - \nabla f(x_*)} + 2\br{\frac{d}{b} - 1}\sqn{h_k - \nabla f(x_*)}.
\end{eqnarray*}
where we used in the first inequality that for all $a, b \in \R^d, \sqn{a + b} \leq 2\sqn{a} + 2\sqn{b}$. Thus, using the fact that $f$ is $L$-smooth, we have
\begin{eqnarray*}
\ec[k]{\sqn{g_k}} \leq \frac{4dL}{b}D_f\br{x_k, x_*} + 2\br{\frac{d}{b} - 1} \sigma_k^2.
\end{eqnarray*}
Moreover,
\begin{eqnarray*}
\ec[k]{\sigma_{k+1}^2} &=& \ec[k]{\sqn{h_{k+1}- \nabla f(x_*)}} = \ec[k]{\sqn{I_{B_k^c}(h_k- \nabla f(x_*)) + I_{B_k}(\nabla f(x_k)- \nabla f(x_*))}} \\
&\overset{\eqref{eq:sqn-sampled-miniblock}}{=}&\br{1 - \frac{b}{d}}\sqn{h_k- \nabla f(x_*)} + \frac{b}{d}\sqn{\nabla f(x_k)- \nabla f(x_*)} \\
&&+ 2\dotprod{I_{B_k^c}(h_k- \nabla f(x_*)), I_{B_k}(\nabla f(x_k)- \nabla f(x_*))} \\
&=& \br{1 - \frac{b}{d}}\sqn{h_k- \nabla f(x_*)} + \frac{b}{d}\sqn{\nabla f(x_k)- \nabla f(x_*)} \\
&& + 2\dotprod{\underbrace{I_{B_k}I_{B_k^c}}_{=0}(h_k- \nabla f(x_*)), \nabla f(x_k)- \nabla f(x_*)}\\
&\leq& \br{1 - \frac{b}{d}}\sqn{h_k  - \nabla f(x_*)} + \frac{2bL}{d}D_f\br{x_k, x_*},
\end{eqnarray*}
where we used in the last inequality the $L-$smoothness of $f$.
\end{proof}

\begin{lemma}\label{lem:constants-b-SEGA}
From Lemma \ref{lem:constants-b-SEGA-source}, we have that the iterates of Algorithm \ref{alg:b-SEGA} satisfy Assumption \ref{asm:main_assumption} and Equation \eqref{eq:def_G} with
\begin{eqnarray}\label{eq:sigma-b-SEGA-app}
\sigma_k^2 = \sqn{h_k - \nabla f(x_*)}
\end{eqnarray}
and constants:
\begin{eqnarray}\label{eq:constants-b-SEGA-app}
A = \frac{2dL}{b}, \; B = 2\br{\frac{d}{b}-1}, \; \rho = \frac{b}{d}, \; C = \frac{bL}{d}, \; D_1 = D_2 = 0, \; G=0.
\end{eqnarray}
\end{lemma}
Using the constant derived in Lemma \ref{lem:constants-b-SEGA} in Theorem \ref{thm:main-prox-dec} gives the following corollary.
\begin{corollary}
    \label{cor:conv-b-SEGA}
    Assume that $f$ satisfies Assumption~\ref{asm:function-class}. Choose for all $k \in \N$,  $\gamma_k = \gamma$, where
    \[ 0 < \gamma < \frac{1}{4(\frac{2d}{b} - 1)L}. \]
     Then, from Theorem \ref{thm:main-prox-dec} and Lemma \ref{lem:constants-b-SEGA}, we have that the iterates given by Algorithm \ref{alg:b-SEGA} verify
    \begin{align}
        \label{eq:conv-b-SEGA}
        \ec{F(\bar{x}_t) - F(x_\ast)} \leq \frac{\sqn{x_0 - x_*} + 2\gamma\br{F(x_0) - F(x_*) +  \frac{2d}{b}\br{\frac{d}{b} - 1}\gamma \sigma^2}}{2\gamma\br{1 - 4\gamma\br{\frac{2d}{b}-1}}t},
    \end{align}
    where $\bar{x}_t \eqdef \frac{1}{t}\sum\limits_{k=0}^{t-1} x_k$.
\end{corollary}

\section{Proofs for Section \ref{sec:unified-analysis}}\label{sec:app-unified-analysis-proofs}

\subsection{Proof of Theorem~\ref{thm:main-prox-dec}}
Before proving Theorem \ref{thm:main-prox-dec}, we present several useful lemmas.
\begin{lemma}[Bounding the gradient variance] Assuming that the $g_k$ are unbiased and that Assumption~\ref{asm:main_assumption} holds, we have
    \begin{equation}
        \label{eq:lma-gradient-variance}
        \ecn{g_k - \nabla f(x_k)}  \leq 2 A D_{f} (x_k, x_\ast) + B \sigma_k^2 + D_1
    \end{equation}
\end{lemma}
\begin{proof}
    Starting from the left hand side of \eqref{eq:lma-gradient-variance}, we have
    \begin{align*}
        \ecn{g_k - \nabla f(x_k)} &= \ecn{ g_k - \nabla f(x_\ast) - \br{ \nabla f(x_k) - \nabla f(x_\ast) } } \\
        &= \ecn{ g_k - \nabla f(x_\ast) - \ec{ g_k - \nabla f(x_\ast) } } \\
        &\leq \ecn{g_k - \nabla f(x_\ast)} \leq 2 A D_{f} (x_k, x_\ast) + B \sigma_k^2 + D_1,
    \end{align*}
    where we used that $\ecn{ X - \ec{X} } = \ecn{X} - \norm{\ec{X}}^2 \leq \ecn{X}$ for any random variable $X$.
\end{proof}

\begin{lemma}
    \label{lemma:atchade-functional-value-bound}
    Suppose that Assumption~\ref{asm:function-class} holds and let $\gamma \in \left (0, \frac{1}{L} \right ]$, then for all $x, y \in \R^d$ and $p = \prox_{\gamma g}(y)$ we have,
    \begin{equation}
        \label{eq:atchade-lemma}
        - 2 \gamma \br{ F(p) - F(x_\ast) } \geq \sqn{p - z } + 2  \ev{p - x_\ast, x - \gamma \nabla f(x) - y} - \sqn{x_\ast - x}.
    \end{equation}
\end{lemma}
\begin{proof}
    We leave the proof to Section~\ref{app:sec:prooflemma}.
\end{proof}

\begin{lemma}
    For any $x \in \R^d$ and minimizer $x_\ast$ of $F$ we have,
    \begin{equation}
        \label{eq:bregman-divergence-bound}
        D_{f} (x, x_\ast) \leq F(x) - F(x_\ast).
    \end{equation}
\end{lemma}
\begin{proof}
    Because $x_\ast$ is a minimizer of $F$ we have that $- \nabla f(x_\ast)  \in \partial R(x_\ast)$. By the definition of subgradients we have
    \[ R(x_\ast) + \ev{ - \nabla f(x_\ast), x - x_\ast } \leq R(x). \]
    Rearranging gives
    \[ -\ev{ \nabla f(x_\ast), x - x_\ast} \leq R(x) - R(x_\ast) \]
    Adding $f(x) - f(x_\ast)$ to both sides we have,
    \[ f(x) - f(x_\ast) - \ev{ \nabla f(x_\ast), x - x_\ast } \leq f(x) + R(x) - \br{ f(x_\ast) + R(x_\ast)} = F(x) - F(x_\ast). \]
   Now note that the on the left hand side we have the Bregman divergence $D_{f} (x, x_\ast)$.
\end{proof}

\begin{definition}
    Given a stepsize $\gamma > 0$, the prox-grad mapping is defined as:
    \begin{equation}
        T_{\gamma} (x) \eqdef \prox_{\gamma R} \br{x - \gamma \nabla f(x)}.
    \end{equation}
\end{definition}

For the ease of exposition, we restate Theorem \ref{thm:main-prox-dec}.

\begin{theorem}
    \label{thm:main-prox-dec-app} 
    Suppose that Assumptions~\ref{asm:main_assumption} and \ref{asm:function-class} hold. Let $M \eqdef B/\rho$ and let $(\gamma_k)_{k\geq 0}$ be a decreasing, strictly positive sequence of step sizes chosen such that
    \[ 0 < \gamma_0 < \frac{1}{2 (A + MC)}. \]
    The iterates given by \eqref{eq:sgd-proximal-iterates} converge according to
    \begin{align}
        \label{eq:main-thm-convergence-bound-app}
        \ec{F(\bar{x}_t) - F(x_\ast)} \leq \frac{V_0 + 2\gamma_0 \delta_0 + 2\br{D_1 + 2 M D_2}\sum_{k=0}^{t-1}\gamma_k^2}{2\sum_{i=0}^{t-1}\br{1 - 2\gamma_i\br{A+MC}} \gamma_i},
    \end{align}
    where $\bar{x}_t \eqdef \sum\limits_{k=0}^{t-1} \frac{\br{ 1 - \gamma_k \eta }  \gamma_k}{\sum_{i=0}^{t-1}\br{1 - \gamma_i\eta} \gamma_i}x_k$ and $V_0 \eqdef \sqn{x_0 - x_\ast} + 2 \gamma_0^2 M \sigma_0^2$ and ${\delta_0 \eqdef F(x_0) - F(x_\ast)}$.
\end{theorem}

\begin{proof}
    Let $x_\ast$ be a minimizer of $F$. Using \eqref{eq:atchade-lemma} from Lemma~\ref{lemma:atchade-functional-value-bound} with $y = x_k - \gamma_k g_k$, $x = x_k$ and $\gamma = \gamma_k$ gives
    \begin{align*}
        -2 \gamma_k \br{ F(x_{k+1}) - F(x_\ast) } &\geq \norm{x_{k+1} - x_\ast}^2 - \norm{x_k - x_\ast}^2 + 2 \gamma_k \ev{ x_{k+1} - x_\ast, g_k - \nabla f(x_k) }.
    \end{align*}
    Multiplying both sides by $-1$ results in
    \begin{align}        
        \label{eq:gen-prox-dec-proof-init}
        2 \gamma_k \br{F(x_{k+1}) - F(x_\ast)} &\leq \sqn{x_k - x_\ast} - \sqn{x_{k+1} - x_\ast} + 2 \gamma_k \ev{x_{k+1} - x_\ast, \nabla f(x_k) - g_k}.
    \end{align}
    Now focusing on the last term in the above and consider the straightforward decomposition
    \begin{align}
        \label{eq:gen-prox-dec-proof-0}
        \ev{x_{k+1} - x_\ast, \nabla f(x_k) - g_k} &= \ev{x_{k+1} - T_{\gamma_k} (x_k), \nabla f(x_k) - g_k } + \ev{T_{\gamma_k} (x_k) - x_\ast, \nabla f(x_k) - g_k}.
     \end{align}
    By Cauchy Schwartz we have that
    \begin{align}
        \label{eq:gen-prox-dec-proof-1}
        \ev{x_{k+1} - T_{\gamma_k} (x_k), \nabla f(x_k) - g_k} \leq \norm{ x_{k+1} - T_{\gamma_k} (x_k) } \norm{g_k - \nabla f(x_k)}.
    \end{align}
    Now using the nonexpansivity of the proximal operator
    \begin{align*}
        \norm{x_{k+1} - T_{\gamma_k} (x_k)} &= \norm{ \prox_{\gamma_k R} \br{ x_k - \gamma_k g_k } - \prox_{\gamma_k R} \br{x_k - \gamma_k \nabla f(x_k)} } \\
        &\leq \norm{ \br{x_k - \gamma_k g_k} - \br{x_k - \gamma_k \nabla f(x_k)} } = \gamma_k \norm{g_k - \nabla f(x_k)}.
    \end{align*}
    Using this in \eqref{eq:gen-prox-dec-proof-1}, we have
    \begin{align}
        \label{eq:gen-prox-dec-proof-2}
        \ev{x_{k+1} - T_{\gamma_k} (x_k), \nabla f(x_k) - g_k} \leq \gamma_k \norm{g_k - \nabla f(x_k)}^2.
    \end{align}
    Using \eqref{eq:gen-prox-dec-proof-2} in \eqref{eq:gen-prox-dec-proof-0} and taking expectation conditioned on $x_k$, and using $\ec[k]{\cdot} \eqdef \ec{\cdot \; | \; x_k}$ for shorthand, we have
    \begin{align}
        \label{eq:gen-prox-dec-proof-2-1}
        \ec[k]{\ev{x_{k+1} - x_\ast, g_k - \nabla f(x_k)}} &\leq \gamma_k \cdot \ecn[k]{g_k - \nabla f(x_k)} + \ev{ T_{\gamma_k} (x_k) - x_\ast, \underbrace{\ec[k]{\nabla f(x_k) - g_k}}_{= 0} } \nonumber \\
        &= \gamma_k \cdot \ecn[k]{g_k - \nabla f(x_k)}.
    \end{align}
	Let $r_k \eqdef x_k-x_\ast.$    Taking expectation conditioned on $x_k$ in \eqref{eq:gen-prox-dec-proof-init} and using \eqref{eq:gen-prox-dec-proof-2-1}, we have
    \begin{align*}
        2 \gamma_k \ec[k]{ F(x_{k+1} - F(x_\ast)) } &\leq \sqn{r_k} - \ecn[k]{r_{k+1}} + 2 \gamma_k^2 \ecn[k]{g_k - \nabla f(x_k)}.
    \end{align*}
    Using \eqref{eq:asm-gradient-opt-distance} from Assumption~\ref{asm:main_assumption} we have
    \begin{align*}
        2 \gamma_k \ec[k]{ F(x_{k+1}) - F(x_\ast) } &\leq \sqn{r_k} - \ecn[k]{r_{k+1}} + 2 \gamma_k^2 \br{ 2 A D_{f} (x_k, x_\ast) + B \sigma_k^2 + D_1}.
    \end{align*}
    Let $V_k \eqdef \sqn{r_k} + 2 M \gamma_k^2 \sigma_{k}^2$ where $M = \frac{B}{\rho}$, then
    \begin{align}
        \label{eq:gen-prox-dec-proof-3}
        \begin{split}
            2 \gamma_k \ec[k]{ F(x_{k+1}) - F(x_\ast) } \leq V_k &- \ec[k]{V_{k+1}} + 4 \gamma_k^2 A D_{f} (x_k, x_\ast) + 2 \gamma_k^2 D_1 \\
            &+ \gamma_k^2 \br{ 2B - 2 M } \sigma_{k}^2 + 2 M \gamma_{k+1}^2 \ec{\sigma_{k+1}^2}.
        \end{split}
    \end{align}
Since $\gamma_{k+1} \leq \gamma_k$ we have that
	\begin{align}
        \label{eq:gen-prox-dec-proof-3.1}
        \begin{split}
            2 \gamma_{k+1} \ec[k]{ F(x_{k+1}) - F(x_\ast) } \leq V_k &- \ec[k]{V_{k+1}} + 4 \gamma_k^2 A D_{f} (x_k, x_\ast) + 2 \gamma_k^2 D_1 \\
            &+ \gamma_k^2 \br{ 2B - 2 M } \sigma_{k}^2 + 2 M \gamma_k^2 \ec{\sigma_{k+1}^2}.
        \end{split}
    \end{align}
    
    Using \eqref{eq:asm-decreasing-noise} from Assumption~\ref{asm:main_assumption}, we have
        \begin{align}
        \label{eq:gen-prox-dec-proof-4}
            2 \gamma_k^2 \br{B - M} \sigma_k^2 + 2 M \gamma_k^2 \ec[k]{\sigma_{k+1}^2} & \leq 2 \gamma_k^2 \br{ B - M + M (1 - \rho) } \sigma_k^2 + 4 M \gamma_k^2 C D_{f} (x_k, x_\ast)\nonumber  \\
            &\quad + 2 M \gamma_k^2 D_2
 \nonumber \\
       & = 2 \gamma_k^2 \underbrace{\br{ B - \rho M}}_{= 0} \sigma_k^2 + 4 M \gamma_k^2 C D_{f} (x_k, x_\ast) +2 M \gamma_k^2 D_2 \nonumber \\
       & \leq 4 M \gamma_k^2 C D_{f} (x_k, x_\ast) +2 M \gamma_k^2 D_2.
    \end{align}
%    \begin{align}
%        \label{eq:gen-prox-dec-proof-4}
%        \begin{split}
%            2 \gamma_k^2 \br{B - M} \sigma_k^2 + 2 M \gamma_k^2 \ec[k]{\sigma_{k+1}^2} \leq 2 \gamma_k^2 \br{ B - M + M (1 - \rho) } \sigma_k^2 &+ 4 M \gamma_k^2 C D_{f} (x_k, x_\ast) \\
%            &+ 2 M \gamma_k^2 D_2
%        \end{split} \nonumber \\
%        = 2 \gamma_k^2 \underbrace{\br{ B - \rho M}}_{= 0} \sigma_k^2 + 4 M \gamma_k^2 C D_{f} (x_k, x_\ast) &+2 M \gamma_k^2 D_2 \nonumber \\
%        \leq 4 M \gamma_k^2 C D_{f} (x_k, x_\ast) &+2 M \gamma_k^2 D_2.
%    \end{align}
    Using \eqref{eq:gen-prox-dec-proof-4} in \eqref{eq:gen-prox-dec-proof-3} gives
    \begin{align}
        \label{eq:gen-prox-dec-proof-5}
        2 \gamma_{k+1} \ec[k]{ F(x_{k+1}) - F(x_\ast) } &\leq V_k - \ec[k]{V_{k+1}} + 2 \gamma_k^2 \br{ 2 A + 2 M C } D_{f} (x_k, x_\ast) + 2 \gamma_k^2 \br{ D_1 + M D_2 }.
    \end{align}
    Let $\eta \eqdef 2A + 2 M C$. Using \eqref{eq:bregman-divergence-bound} in \eqref{eq:gen-prox-dec-proof-5} we have,
    \begin{align*}
        2 \gamma_{k+1} \ec[k]{ F(x_{k+1}) - F(x_\ast) } &\leq V_k - \ec[k]{V_{k+1}} + 2 \gamma_k^2 \eta \br{ F(x_k) - F(x_\ast) } + 2 \gamma_k^2 \br{ D_1 + M D_2 }.
    \end{align*}
    Using the abbreviation $\delta_k = F(x_k) - F(x_\ast)$ gives
    \begin{align*}
        2 \gamma_{k+1} \ec[k]{\delta_{k+1}} &\leq V_{k} - \ec[k]{V_{k+1}} + 2 \gamma_k^2 \eta \delta_k + 2 \gamma_k^2 \br{D_1 + M D_2}.
    \end{align*}
    Taking expectation,
   \begin{align*}
        2 \gamma_{k+1} \ec{\delta_{k+1}} &\leq \ec{V_{k}} - \ec{V_{k+1}} + 2 \gamma_k^2 \eta \ec{\delta_k} + 2 \gamma_k^2 \br{D_1 + M D_2},
    \end{align*}
 summing over $k =0,\ldots, t-1$ and using telescopic cancellation gives
    \begin{align*}
        2  \sum_{k=1}^{t}\gamma_k\ec{\delta_{k}} &\leq V_{0} - \ec{V_{t}} + 2  \eta \sum_{k=0}^{t-1}\gamma_k^2 \ec{\delta_k} + 2 \br{ D_1 + M D_2 } \sum_{k=0}^{t-1}\gamma_k^2.
    \end{align*}
Adding $2\gamma_0\delta_0$ to both sides of the above inequality and rearranging,
\begin{align*}
2  \sum_{k=0}^{t-1}\gamma_k(1-\eta\gamma_k)\ec{\delta_{k}} &\leq V_{0} - \ec{V_{t}} + 2 \gamma_0 \delta_0 + 2 \br{ D_1 + M D_2 } \sum_{k=0}^{t-1}\gamma_k^2
\end{align*}

    where we also used that $V_t \geq 0$ and $\delta_t \geq 0.$
   
    By the choice of $\gamma_0$ we have $1 - \gamma_0 \eta > 0$, and since $(\gamma_i)_i$ is a decreasing sequence, we have $1 - \gamma_i \eta > 0$ for all $i$. Hence dividing both sides by $2\sum\limits_{i=0}^{t-1}\br{1 - \gamma_i\eta} \gamma_i$, we have
    \begin{align*}
        \sum_{k=0}^{t-1} w_k \ec{\delta_k} \leq \frac{V_0 + 2 \gamma_0 \delta_0}{2\sum_{i=0}^{t-1}\br{1 - \gamma_i\eta} \gamma_i} + \br{D_1 + 2 M D_2} \frac{\sum_{k=0}^{t-1}\gamma_k^2}{\sum_{i=1}^{t}\br{1-\gamma_i\eta}\gamma_i},
    \end{align*}
    where $w_k \eqdef \frac{\br{ 1 - \gamma_k \eta }  \gamma_k}{\sum_{i=0}^{t-1}\br{1 - \gamma_i\eta} \gamma_i}$ for all $k \in \left\{0,\dots,t-1\right\}$. Note that $\sum_{k=0}^{t-1} w_k = 1$ and $w_k \geq 0$ for all $k \in \left\{0,\dots,t-1\right\}$. Hence, since $F$ is convex, we can use Jensen's inequality to conclude
    \begin{align*}
        \ec{F(\bar{x}^k) - F(x_\ast)} &= \ec{F\br{ \sum\limits_{k=0}^{t-1} w_k x_k} - F(x_\ast)} \\
        &\leq \sum_{k=0}^{t-1}w_k \ec{\delta_k} \leq \frac{V_0 + 2\gamma_0 \delta_0}{2\sum_{i=0}^{t-1}\br{1 - \gamma_i\eta} \gamma_i} + \frac{\br{D_1 + 2 M D_2} \sum_{k=0}^{t-1}\gamma_k^2}{\sum_{i=0}^{t-1}\br{1-\gamma_i\eta}\gamma_i}.
    \end{align*}
    Writing out the definition of $\delta_0$ yields the theorem's statement.
\end{proof}

\section{Proofs for Section \ref{sec:main-corollaries}}\label{sec:app-main-corollaries-proofs}
\subsection{Proof of Corollary \ref{cor:conv-dec}}
\begin{proof}

Note that, using the integral bound, we have:
\begin{eqnarray*}
\sum_{k=0}^{t-1}\gamma_k^2 &\leq& \gamma^2\br{\log(t) + 1}\\
\sum_{k=0}^{t-1}\gamma_k &\geq& 2\gamma\br{\sqrt{t} - 1}
\end{eqnarray*}
Moreover, note that since $\gamma_k \leq \frac{1}{4\br{A+MC}}$, we have $1 - 2\gamma_k(A+MC) \geq \frac{1}{2}$ for all $k \in \N$. Thus
\begin{equation*}
\sum_{k=0}^{t-1} \frac{1}{2\gamma_k\br{1 - \eta \gamma_k}} \leq \frac{1}{2\gamma \br{\sqrt{t} - 1}}.
\end{equation*}
Corollary \ref{cor:conv-dec} follows from using these bounds in Equation \eqref{eq:main-thm-convergence-bound}.
\end{proof}

\section{Proofs for Section \ref{sec:optimal-minibatch} }\label{sec:app-optimal-minibatch-proofs}
\subsection{Proof of Proposition \ref{prop:conv-vr-smooth}}
\begin{proof}
\begin{eqnarray}
\sqn{x_{k+1} - x_*} &=& \sqn{x_k - x_*} - 2\gamma\dotprod{g_k, x_k - x_*} +  \gamma^2\sqn{g_k}.
\end{eqnarray}
Thus, taking expectation conditioned on $x_k$, and using $\ec[k]{\cdot} \eqdef \ec{\cdot \; | \; x_k}$ for shorthand, we have
\begin{eqnarray*}
\ec[k]{\sqn{x_{k+1} - x_*}} &=& \sqn{x_k - x_*} - 2\gamma\dotprod{\nabla f(x_k), x_k - x_*} +  \gamma^2\ec[k]{\sqn{g_k}}\\
&\overset{\eqref{eq:convex}+\eqref{eq:asm-gradient-unbiased}+\eqref{eq:asm-gradient-opt-distance}}{\leq}& \sqn{x_k - x_*} - 2\gamma(1-2\gamma A)\br{f(x_k) - f(x_*)} + B\sigma_k^2.
\end{eqnarray*}
Thus, using \eqref{eq:asm-decreasing-noise},
\begin{eqnarray*}
\ec[k]{\sqn{x_{k+1} - x_*}} + 2M\gamma^2\ec[k]{\sigma_{k+1}^2} \leq \sqn{x_k - x_*} - 2\gamma(1-2\gamma(A+MC))\br{f(x_k) - f(x_*)} + 2M\gamma^2\sigma_k^2.
\end{eqnarray*}
Thus, rearranging and taking the expectation, we have:
\begin{eqnarray*}
2\gamma(1-2\gamma(A+MC))\ec{f(x_k) - f(x_*)} &\leq& \ecn{x_k - x_*} - \ecn{x_{k+1} - x_*} \\
&& + 2M\gamma^2\br{\ec{\sigma_k^2} - \ec{\sigma_{k+1}^2}}.
\end{eqnarray*}
Summing over $k =0,\ldots, t-1$ and using telescopic cancellation gives
\begin{eqnarray*}
2\gamma(1-2\gamma(A+MC))\sum_{k=0}^{t-1}\ec{f(x_k) - f(x_*)} &\leq& \sqn{x_0 - x_*} - \ecn{x_{k} - x_*} \\
&& + 2M\gamma^2\br{\ec{\sigma_0^2} - \ec{\sigma_{k+1}^2}}.
\end{eqnarray*}
Ignoring the negative terms in the upper bound, and using Jensen's inequality, we have
\begin{eqnarray*}
\ec{f(\bar{x}_t) - f(x_*)} &\leq& \frac{\sqn{x_0 - x_*} + 2M\gamma^2\sigma_0^2}{2\gamma(1-2\gamma(A+MC))t}.
\end{eqnarray*}
Moreover, notice that if $\gamma \leq \frac{1}{4(A+MC)}$, then $2(1-2\gamma(A+MC)) \geq 1$, which gives \eqref{eq:thm-convergence-bound-smooth-nice}.
\end{proof}

\subsection{Optimal minibatch size for $b$-SAGA (Algorithm \ref{alg:b-SAGA})}\label{sec:app-optimal-minibatch-b-SAGA}
In this Section, we present the proofs for Section \ref{sec:optimal-b-SAGA}.
\subsubsection{Proof of Lemma \ref{lem:constants-b-SAGA-main} }
\begin{proof}
For constant $A, B, \rho, C, D_1, D_2$, see Lemma \ref{lem:constants-b-SAGA}.
Moreover,
\begin{eqnarray*}
\sigma_0^2 &=&  \frac{1}{nb} \frac{n-b}{n-1}\trn{\nabla H(x_0) - \nabla H(x_*)} = \frac{1}{nb} \frac{n-b}{n-1} \sum_{i=1}^n \sqn{\nabla f_i(x_0) - \nabla f_i(x_*)}\\
&=& \frac{1}{b} \frac{n-b}{n-1} \frac{1}{n}\sum_{i=1}^n \sqn{\nabla f_i(x_0) - \nabla f_i(x_*)} \\
&\overset{\eqref{eq:convandsmooth_sum}}{\leq}& \frac{1}{b} \frac{n-b}{n-1}L_{\max} \br{f(x_0) - f(x_*)} \\
&\overset{\eqref{eq:Lsmooth} + \eqref{eq:sigma-b-SAGA}}{\leq}&  \zeta(b) L \sqn{x_0 - x_*}.
\end{eqnarray*}
Thus, \eqref{eq:def_G} holds with $G = \zeta(b)$.
\end{proof}

\subsubsection{Proof of Proposition \ref{prop:optimal-minibatch-SAGA-main}}
%Since the computations for the optimal minibatch size were quite heavy, they were done with the help of Maple. We report the corresponding code in Section \ref{sec:maple-optimal-b-SAGA}.
\begin{proof}
First, since $\frac{\sqn{x_0 - x_*}}{\epsilon}$ does not depend on $b$, the variations of $K(b)$ are the same as those of
\begin{eqnarray}
Q(b) = \frac{4\br{3(n-b)L_{\max} + 2n(b-1)L}}{b(n-1)} + \frac{n(n-b)L_{\max}L}{2b\br{3(n-b)L_{\max} + 2n(b-1)L}}. 
\end{eqnarray}

Let's determine the sign of $Q^{'}(b)$. We have:
\begin{eqnarray}
Q^{'}(b) = \frac{W_1 b^2 + W_2 b + W_3}{4(n-1)\br{\br{2nL - 3L_{\max}}b + \br{\frac{3L_{\max}}{2} - L}n}^2},
\end{eqnarray}
where
\begin{eqnarray*}
W_1 &=& 4\br{2nL - 3L_{\max}}^3, \\
W_2 &=& 8n\br{3L_{\max} - 2L}\br{2nL - 3L_{\max}}^2, \\
W_3 &=& n^2\br{108L_{\max}^3 + 72L\br{n+2L}L_{\max}^2 - \br{n^2 + 94n + 49}L^2L_{\max} + 32nL^3}.
\end{eqnarray*}
And we have:
\begin{eqnarray}
W_2^2 - 4W_1W_2 = 16n^2(n-1)^2L^2L_{\max}\br{2nL - 3L_{\max}}^3.
\end{eqnarray}

\paragraph{Case 1: $L_{\max} > \frac{2nL}{3}$.} We have $2nL - 3L_{\max} < 0$. Hence, $W_2^2 - 4W_1W_2 < 0$.

Moreover, since $W_1 < 0$, we have
\begin{eqnarray}
L_{\max} > \frac{2nL}{3} \implies K'(b) < 0.
\end{eqnarray}
Thus,
\begin{eqnarray}
\boxed{L_{\max} > \frac{2nL}{3} \implies b^* = n}
\end{eqnarray}

\paragraph{Case 2: $L_{\max} \leq \frac{2nL}{3}$.} Then, $W_2^2 - 4W_1W_2 \geq 0$ and $K'(b) = 0$ has at least one solution. We are now going to examine wether or not $K(b)$ is convex. We have:
\begin{eqnarray}
Q^{''}(b) = \frac{2n^2(n-1)L_{\max}L^2\br{2nL - 3L_{\max}}}{\br{\br{2nL - 3L_{\max}}b + \br{3L_{\max} - 2L}n}^3} \geq 0.
\end{eqnarray}
Thus, $K(b)$ is convex. $K^{'}(b) = 0$ has two solutions:
%\begin{eqnarray}
%b_1 &=& \frac{n}{2}\br{-\br{2nL - 3L_{\max}}\br{3L_{\max} - 2L} + (n-1)L\sqrt{2nL - 3L_{\max}}}, \\
%b_2 &=& \frac{n}{2}\br{-\br{2nL - 3L_{\max}}\br{3L_{\max} - 2L} - (n-1)L\sqrt{2nL - 3L_{\max}}}.
%\end{eqnarray}
\begin{eqnarray}
b_1 &=& \frac{n\br{(n-1)L\sqrt{L_{\max}} - 2\sqrt{2nL - 3L_{\max}}(3L_{\max} - 2L)} }{2(2nL - 3L_{\max})^{\frac{3}{2}}}, \\
b_2 &=& \frac{-n\br{(n-1)L\sqrt{L_{\max}} + 2\sqrt{2nL - 3L_{\max}}(3L_{\max} - 2L)} }{2(2nL - 3L_{\max})^{\frac{3}{2}}}.
\end{eqnarray}
But since $b_2 \leq 0$, we have that:
\begin{eqnarray}
\boxed{L_{\max} \leq \frac{2nL}{3} \implies b^* = \left\{
      \begin{array}{ll}
          1 & \mbox{if } b_1 < 2 \\
          \floor{b_1} & \mbox{if } 2 \leq b_1 < n \\
          n & \mbox{if } b_1 \geq n
      \end{array} \right.}
\end{eqnarray}

\end{proof}

\subsection{Optimal minibatch size for $b$-L-SVRG (Algorithm \ref{alg:b-L-SVRG})}\label{sec:app-optimal-minibatch-b-L-SVRG}
In this section, we present a detailed analysis of the optimal minibatch size derived in Section \ref{sec:optimal-b-L-SVRG}.

\begin{lemma}\label{lem:constants-b-L-SVRG-main}
We have that the iterates of Algorithm \ref{alg:b-L-SVRG} satisfy Assumption \ref{asm:main_assumption} and Equation \eqref{eq:def_G} with 
\begin{eqnarray}\label{eq:sigma-b-L-SVRG}
\sigma_k^2 = \ecb{\sqn{\nabla f_{B}(w_k) - \nabla f_{B}(x_*) - (\nabla f(w_k) - \nabla f(x_*))}},
\end{eqnarray}
and constants
\begin{eqnarray}\label{eq:constants-b-L-SVRG}
A = 2\cL(b), \; B = 2, \; \rho = p, \; C = p\cL(b), \; D_1 = D_2 = 0, \; G = \cL(b)L,
\end{eqnarray}
where $\cL(b)$ is defined in \eqref{eq:cL-b-nice}.
\end{lemma}
\begin{proof}
For constant $A, B, \rho, C, D_1, D_2$, see Lemma \ref{lem:constants-L-SVRG-AS} and Corollary \ref{cor:constants-b-L-SVRG}.

Moreover,
\begin{eqnarray*}
\ec{\sqn{\nabla f_{v_0}(x_0) - \nabla f_{v_0}(x_*) - \br{\nabla f(x_0) - \nabla f(x_*)}}} &\leq& \ec{\sqn{\nabla f_{v_0}(x_0) - \nabla f_{v_0}(x_*)}} \\
&\overset{\eqref{eq:expected-smoothness}}{\leq}& 2\cL(b)D_f\br{x_0, x_*} \\
&\overset{\eqref{eq:Lsmooth}}{\leq}& \cL(b) L\sqn{x_0 - x_*}.
\end{eqnarray*}
where we used in the first inequality that $\ecn{ X - \ec{X} } = \ecn{X} - \norm{\ec{X}}^2 \leq \ecn{X}$.
Thus, \eqref{eq:def_G} holds with $G = \cL(b)L$.
\end{proof}

In the next corollary, we will give the iteration complexity for Algorithm \ref{alg:b-L-SVRG} in the case where $p = 1/n$, which is the usual choice for $p$ in practice. A justification for this choice can be found in \cite{Kovalev20, Sebbouh19}.
\begin{corollary}[Iteration complexity of L-SVRG]\label{cor:complexity_l-svrg}
Consider the iterates of Algorithm \ref{alg:b-L-SVRG}. Let $p=1/n$ and $\gamma = \frac{1}{12\cL(b)}$. Given the constants obtained for Algorithm \ref{alg:b-L-SVRG} in \eqref{eq:constants-b-L-SVRG}, we have, using Corollary \ref{cor:complexity-vr-methods}, that if
\begin{eqnarray}
k \geq \br{12\cL(b) + \frac{nL}{6}}\frac{\sqn{x_0 - x_*}}{\epsilon},
\end{eqnarray}
then, $\ec{f(\bar{x}_k) - f(x_*)} \leq \epsilon$.
\end{corollary}

The usual definition for the total complexity is the expected number of gradients computed per iteration, times the iteration complexity, required to reach an $\epsilon-$approximate solution in expectation. However, since L-SVRG computes the full gradient every $n$ iterations in expectation, we can say that L-SVRG computes roughly $2b + 1$ gradients every iteration, so that after $n$ iteration, it will have computed $n + 2bn$ gradient. Thus, the total complexity for SVRG is:
\begin{eqnarray}\label{eq:total-comp-b-L-SVRG}
K(b) &\eqdef& \br{1+2b}\br{12\cL(b) + \frac{nL}{6}}\frac{\sqn{x_0 - x_*}}{\epsilon} \\
&=& \br{1+2b}\br{\frac{12\br{3(n-b)L_{\max} + 2n(b-1)L}}{b(n-1)}+\frac{nL}{6}}\frac{\sqn{x_0 - x_*}}{\epsilon}.
\end{eqnarray}

\subsubsection{Proof of Proposition \ref{prop:optimal-minibatch-b-L-SVRG}}
\begin{proof}
Since the factor $\frac{\sqn{x_0 - x_*}}{\epsilon}$ which appears in \eqref{eq:total-comp-b-L-SVRG} does not depend on the minibatch size,  minimizing the total complexity in the minibatch size corresponds to minimizing the following quantity:

\begin{equation}
Q(b) = \frac{12\br{3(n-b)L_{\max} + 2n(b-1)L}}{b(n-1)}+\frac{nL}{6}.
\end{equation}
We have
\begin{eqnarray*}
(n-1)Q(b) &=& 12(n-1)\cL(b) + 24(n-1)b\cL(b) + \frac{n(n-1)L b}{3} + \frac{nL}{6}\\ %3n(n-1)Lb + \frac{n^2(n-1)L}{6} \\
&=& \frac{12n(L_{\max} - L)}{b} + \br{24(nL - L_{\max}) + \frac{n(n-1)L}{3}}b + \xi,
\end{eqnarray*}
where $\xi$ is a constant independent of $b$. Differentiating, we have:
\begin{eqnarray*}
(n-1)Q'(b) = -\frac{12n(L_{\max} - L)}{b^2} + 24(nL - L_{\max}) + \frac{n(n-1)L}{3}.
\end{eqnarray*}
Since $L_{\max} \geq L$ and $nL \geq L_{\max}$ (see for example Lemma A.6 in \cite{Sebbouh19}), $C(b)$ is a convex function of $b$. Thus, $Q(b)$ is minimized when $Q'(b) = 0$. Hence:
\begin{eqnarray}
\boxed{b^* = 6\sqrt{\frac{n\br{L_{\max} - L}}{72\br{nL-L_{\max}} + n(n-1)L}}}
\end{eqnarray}
Since $L_{\max}$ can take any value in the interval $[L, nL]$, we have $b^* \in [0, 6]$.
\end{proof}

\subsection{Optimal miniblock size for $b$-SEGA (Algorithm \ref{alg:b-SEGA})}\label{sec:app-optimal-minibatch-b-SEGA}
In this section, we define for any $j \in [d]$ the matrix $I_j \in \R^{d\times d}$ such that 
\begin{equation}
(I_j)_{pq} \eqdef \left\{
      \begin{array}{ll}
          1 & \mbox{if }p=q=j \\
          0 & \mbox{otherwise }
      \end{array} \right.,
\end{equation}
and we consequently define for any subset $B \subseteq [d]$,
\begin{eqnarray}
I_B \eqdef \sum_{j \in B} I_j
\end{eqnarray}

\begin{algorithm}[h]
  \begin{algorithmic}%[1]
    \State \textbf{Parameters} step size $\gamma$, block size $b \in [d]$
    \State \textbf{Initialization}   $x_0 \in \mathbb{R}^d$, $h^0 = 0$
    \For {$k=1, 2,\dots$}\vskip 1ex
	  	      
      \State Sample a miniblock $B_k \subseteq [d]$ s.t. $|B_k| = d$
      \State $h^{k+1} = h^k + I_{B_k}\br{\nabla f(x_k) - h_k}$
      \State $g_k = \frac{d}{b}I_{B_k}\br{\nabla f(x_k) - h_k} + h_k$
      \State $x_{k+1} = \prox_{\gamma R}\br{x_k - \gamma g_k}$
    \EndFor

  \end{algorithmic}
  \caption{$b$-SEGA}
  \label{alg:b-SEGA}
\end{algorithm}

\begin{corollary}
From Lemma \ref{lem:constants-b-SEGA}, we have that the iterates of Algorithm \ref{alg:b-SEGA} satisfy Assumption \ref{asm:main_assumption} and Equation \eqref{eq:def_G} with
\begin{eqnarray}\label{eq:sigma-b-SEGA}
\sigma_k^2 = \sqn{h_k - \nabla f(x_*)}
\end{eqnarray}
and constants:
\begin{eqnarray}\label{eq:constants-b-SEGA}
A = \frac{2dL}{b}, \; B = 2\br{\frac{d}{b}-1}, \; \rho = \frac{b}{d}, \; C = \frac{bL}{d}, \; D_1 = D_2 = 0, \; G=0.
\end{eqnarray}
\end{corollary}
\begin{proof}
For the constants $A, B, \rho, C, D_1, D_2$, see Lemma \ref{lem:constants-b-SEGA}. Moreover, in Algorithm \ref{alg:b-SEGA}, $h_0 = 0$. Thus, $\sigma_0^2 = \sqn{h_0} = 0$. Thus, \eqref{eq:def_G} holds with $G = 0$.
\end{proof}

In the next corollary, we will give the iteration complexity for Algorithm \ref{alg:b-SEGA}.
\begin{corollary}[Iteration complexity of b-SEGA]\label{cor:complexity-b-SEGA}
Consider the iterates of Algorithm \ref{alg:b-SEGA}. Let $\gamma = \frac{b}{4\br{3d - b}L}$. Given the constants obtained for Algorithm \ref{alg:b-SEGA} in \eqref{eq:constants-b-SEGA}, we have, using Corollary \ref{cor:complexity-vr-methods}, that if
\begin{eqnarray}
k \geq \frac{4(3d-b)L}{b}\frac{\sqn{x_0 - x_*}}{\epsilon},
\end{eqnarray}
then, $\ec{F(\bar{x}_k) - F(x_*)} \leq \epsilon$.
\end{corollary}

Here, we define the total complexity as the number of coordinates of the gradient that we sample at each iteration times the iteration complexity. Since at each iteration, we sample $b$ coordinates of the gradient, the total complexity for Algorithm \ref{alg:b-SEGA} to reach an $\epsilon-$approximate solution is
\begin{eqnarray}\label{eq:total-comp-b-SEGA}
K(b) &\eqdef& 4\br{3d-b}L\frac{\sqn{x_0 - x_*}}{\epsilon}.
\end{eqnarray}
Thus, we immediately have the following proposition.
\begin{proposition}
Let $b^* = \underset{b \in [d]}{\argmin}\, K(b)$, where $K(b)$ is defined in \eqref{eq:total-comp-b-SEGA}. Then,
\begin{eqnarray}\label{eq:optimal_miniblock_b-SEGA}
b^* = d.
\end{eqnarray}
\end{proposition}
The consequence of this proposition is that when using Algorithm \ref{alg:b-SEGA}, one should always use as big a miniblock as possible if the cost of a single iteration is proportional to the miniblock size.

\section{Auxiliary Lemms}\label{sec:auxiliary-lemmas}
\subsection{Smoothness and Convexity Lemma}

We now develop an immediate consequence of each $f_i$ being convex and smooth based on the follow lemma.

\begin{lemma}\label{lem:smoothconvexaroundxst} 
Let  $g: \R^d \mapsto \R$ be a convex function
\begin{align}
g(z) -g(x) & \leq  \dotprod{\nabla g(z), z-x}, \quad  \forall x,z\in\R^d,\label{eq:conv}
\end{align}
and $L_g$--smooth 
\begin{align} \label{eq:smoothnessfuncstar}
 g(z) -g(x) &\leq  \dotprod{\nabla g(x), z-x} +\frac{L_g}{2}\norm{z-x}_2^2, \quad  \forall x,z\in\R^d.
\end{align}
It follows  that
\begin{equation}\label{eq:convandsmooth}
 \norm{\nabla g(x) - \nabla g(z)}^2  \leq L_g (g(x) -g(z) - \dotprod{\nabla g(z), x-z}), \quad \forall x\in\R^d.
\end{equation}
 \end{lemma}
\begin{proof}
Fix $i\in \{1,\ldots, n\}$ and let \begin{equation}
z = x - \frac{1}{L_g}(\nabla g(x) -\nabla g(x^*)).
\end{equation}. To prove~\eqref{eq:convandsmooth}, it follows that
\begin{eqnarray}
g(x^*) -g(x) & = & g(x^*) -g(z)+g(z) - g(x)\nonumber\\ 
&\overset{\eqref{eq:conv}+\eqref{eq:smoothnessfuncstar} } \leq &
\dotprod{\nabla g(x^*), x^*-z} + \dotprod{\nabla g(x), z-x} +\frac{L_g}{2}\norm{z-x}_2^2.\label{eq:tempsanuin}
\end{eqnarray}

Substituting this in $z$ into~\eqref{eq:tempsanuin} gives
\begin{eqnarray}
g(x^*) -g(x) & = &
\dotprod{\nabla g(x^*), x^*-x + \frac{1}{L_g}(\nabla g(x) -\nabla g(x^*))} - \frac{1}{L_g}\dotprod{\nabla g(x), \nabla g(x) -\nabla g(x^*)} \nonumber \\
& & \quad +\frac{1}{2L_g}\norm{\nabla g(x) -\nabla g(x^*)}_2^2 \nonumber \\
& =&\dotprod{\nabla g(x^*), x^*-x}   - \frac{1}{L_g}\norm{\nabla g(x)-\nabla g(x^*)}_2^2 +\frac{1}{2L_g}\norm{\nabla g(x) -\nabla g(x^*)}_2^2 \nonumber \\
&= & \dotprod{\nabla g(x^*), x^*-x}   - \frac{1}{2L_g}\norm{\nabla g(x)-\nabla g(x^*)}_2^2.\nonumber 
\end{eqnarray}
\end{proof}

\begin{lemma}
Suppose that for all $i \in [n]$, $f_i$ is convex and $L_i-$ smooth, and let $L_{\max} = \max_{i \in [n]} L_i$. Then
\begin{eqnarray}\label{eq:convandsmooth_sum}
\frac{1}{n}\sum_{i=1}^n \sqn{\nabla f_i(x) - \nabla f_i(x_*)} \leq 2L_{\max}\br{f(x) - f(x_*)}.
\end{eqnarray}
\end{lemma}
\begin{proof}
From \eqref{eq:convandsmooth}, we have for all $i \in [n]$, 
\begin{eqnarray}
\sqn{\nabla f_i(x) - \nabla f_i(x_*)} &\leq& 2L_i \br{f(x) - f(x_*) - \dotprod{\nabla f_i(x_*), x - x_*}}
\end{eqnarray} 
Thus,
\begin{eqnarray}
\frac{1}{n}\sum_{i=1}^n\sqn{\nabla f_i(x) - \nabla f_i(x_*)} &\leq& 2L_{\max} \br{f(x) - f(x_*) - \dotprod{\nabla f(x_*), x - x_*}}\\
&=& 2L_{\max} \br{f(x) - f(x_*)}.
\end{eqnarray}
\end{proof}

\subsection{Proximal Lemma}
\begin{lemma} Let $R: \R^d \mapsto \R$ be a convex lower semi-continuous function. For $z,y \in \R^d$ and $\gamma>0$. With $p = \prox_{\gamma g} (y)$ we have that for
\begin{equation} \label{eq:2ndproxtheo}
g(p) - g(z) \leq -\frac{1}{\gamma} \dotprod{p - y, p -z}.
\end{equation}
\end{lemma}
\begin{proof}
This is classic result, see for example the ``Second Prox Theorem'' in Section 6.5 in~\cite{beck2017first}.
%See Propositions 4.2., 12.26 and 12.27 in~\cite{combetteeBook2011}
\end{proof}

\subsection{Proof of Lemma~\ref{lemma:atchade-functional-value-bound}}
\label{app:sec:prooflemma}
%Here
%\begin{lemma}
% \[  F(p) - F(z)  \leq -\frac{1}{2\gamma}\sqn{p - z } - \frac{1}{\gamma}  \ev{x - \gamma \nabla f(x) - y, p-z} + \frac{1}{2\gamma}\sqn{z - x}. \]
%\end{lemma}

This proof is based on the proof of  Lemma~8 in \cite{Atchade17}. The only difference is that in  \cite{Atchade17} the authors assume that $f$ is convex.
Indeed, using the convexity of $f$
 \[f(x) -f(x_\ast) \geq  - \langle \nabla f(x), x_\ast - x \rangle\]
%\[f(x)- f(z) \leq  - \dotprod{\nabla f(x), z-x} \]
in combination with~\eqref{eq:2ndproxtheo} where $z = x_\ast$ gives
\begin{align*}
f(x)+g(p)- F(x_\ast) & \leq -\frac{1}{\gamma} \dotprod{p - y, p -x_\ast}- \dotprod{\nabla f(x), x_\ast-x}. 
\end{align*}
Now using smoothness
\[ f(p) -f(x) \leq \dotprod{\nabla f(x),p -x} + \frac{1}{2\gamma} \norm{p-x}^2,\]
gives
%& \leq \frac{1}{\gamma} \dotprod{ y-p +\gamma \nabla f(x) , p -z}- \dotprod{\nabla f(x), p-x}+\dotprod{\nabla f(x),p -x} + \frac{1}{2\gamma} \norm{p-x}^2\\
\begin{align}
F(p)- F(x_\ast) 
&\leq-\frac{1}{\gamma} \dotprod{p - y, p -x_\ast}- \dotprod{\nabla f(x), x_\ast-x} +\dotprod{\nabla f(x),p -x} + \frac{1}{2\gamma} \norm{p-x}^2 \nonumber \\
&= -\frac{1}{\gamma} \dotprod{p - y, p -x_\ast}+\dotprod{\nabla f(x),p -x_\ast} + \frac{1}{2\gamma} \norm{p-x}^2 \nonumber \\
&= -\frac{1}{\gamma} \dotprod{p -\gamma \nabla f(x)- y, p -x_\ast} + \frac{1}{2\gamma} \norm{p-x}^2 \nonumber  \\
&= -\frac{1}{\gamma} \dotprod{p-x+x -\gamma \nabla f(x)- y, p -x_\ast} + \frac{1}{2\gamma} \norm{p-x}^2 \nonumber \\
&= -\frac{1}{\gamma} \dotprod{p-x, p -x_\ast} -\frac{1}{\gamma} \dotprod{x -\gamma \nabla f(x)- y, p -x_\ast} + \frac{1}{2\gamma} \norm{p-x}^2 . \label{eq:tenaoineria}
\end{align}
Using that
\begin{equation}
-2 \dotprod{p-x, p -x_\ast}  + \norm{p-x}^2 = -\norm{p-x_\ast}^2 + \norm{z-x}^2,
\end{equation}
in combination with~\eqref{eq:tenaoineria} gives 
 \[  F(p) - F(x_\ast)  \leq -\frac{1}{2\gamma}\sqn{p - x_\ast} - \frac{1}{\gamma}  \ev{x - \gamma \nabla f(x) - y, p-x_\ast} + \frac{1}{2\gamma}\sqn{x_\ast - x}. \]
 Now it remains to multiply both sides by $-2\gamma$ to arrive at~\eqref{eq:atchade-lemma}.

%\section{Maple Code for Section \ref{sec:optimal-b-SAGA}}\label{sec:maple-optimal-b-SAGA}

%\begin{figure*}[h]
%\begin{subfigure}[b]{\textwidth}
%\includegraphics[width=\textwidth]{maple_def.png}
%\end{subfigure}\\
%\begin{subfigure}[b]{\textwidth}
%\includegraphics[width=\textwidth]{maple_diff_1.png}
%\end{subfigure}\\
%\begin{subfigure}[b]{\textwidth}
%\includegraphics[width=\textwidth]{maple_disc.png}
%\end{subfigure}\\
%\begin{subfigure}[b]{\textwidth}
%\includegraphics[width=\textwidth]{maple_diff_2.png}
%\end{subfigure}\\
%\begin{subfigure}[b]{\textwidth}
%\includegraphics[width=\textwidth]{maple_solutions.png}
%\end{subfigure}
%\end{figure*}

\end{appendices}

\end{document}

%% file: header.tex
\newcommand\ec[2][]{\ensuremath{\mathbb{E}_{#1} \left[#2\right]}}
\newcommand\ecb[2][]{\ensuremath{\mathbb{E}_{B_{#1}} \left[#2\right]}}
\newcommand\ecn[2][]{\ec[#1]{\norm{#2}^2}}
\newcommand\ecd[2][]{\ensuremath{\mathbb{E}_{\mathcal{D}} \left[#2\right]}}

\newcommand\ev[1]{\left \langle #1 \right \rangle}
\newcommand\dotprod[1]{\left \langle #1 \right \rangle}
\newcommand\br[1]{\left ( #1 \right )}
\newcommand\pbr[1]{\left \{ #1 \right \} }
\newcommand\floor[1]{\left \lfloor #1 \right \rfloor}

\newcommand{\1}{\mathbbm{1}}
\newcommand{\N}{\mathbb{N}}

\newcommand{\R}{\mathbb{R}}

\DeclareMathOperator*{\argmin}{\arg\!\min}

\newcommand{\diag}{\mathrm{diag }}

\newcommand{\Proba}{\mathbb{P}}

\newcommand{\norm}[1]{\left\lVert#1\right\rVert}
\newcommand{\sqn}[1]{{\left\lVert#1\right\rVert}^2}
\newcommand{\trn}[1]{{\left\lVert#1\right\rVert}^2_{\mathrm{Tr}}}

\newcommand{\tr}{\mathrm{Tr}}

\newcommand{\prox}{\mathrm{prox}}
\renewcommand{\L}{\mathcal{L}}
\newcommand{\D}{\mathcal{D}}

\newcommand{\eqdef}{\overset{\text{def}}{=}}
\makeatletter
\newsavebox\myboxA
\newsavebox\myboxB
\newlength\mylenA

\newcommand{\cL}{\mathcal{L}}

\usepackage{tcolorbox}
\usepackage{pifont}
\definecolor{mydarkgreen}{RGB}{39,130,67}

\definecolor{mydarkred}{RGB}{192,47,25}

\newcommand*\overbar[2][0.75]{%
    \sbox{\myboxA}{$\m@th#2$}%
    \setbox\myboxB\null% Phantom box
    \ht\myboxB=\ht\myboxA%
    \dp\myboxB=\dp\myboxA%
    \wd\myboxB=#1\wd\myboxA% Scale phantom
    \sbox\myboxB{$\m@th\overline{\copy\myboxB}$}%  Overlined phantom
    \setlength\mylenA{\the\wd\myboxA}%   calc width diff
    \addtolength\mylenA{-\the\wd\myboxB}%
    \ifdim\wd\myboxB<\wd\myboxA%
       \rlap{\hskip 0.5\mylenA\usebox\myboxB}{\usebox\myboxA}%
    \else
        \hskip -0.5\mylenA\rlap{\usebox\myboxA}{\hskip 0.5\mylenA\usebox\myboxB}%
    \fi}
\makeatother

\usepackage[colorinlistoftodos,bordercolor=orange,backgroundcolor=orange!20,linecolor=orange,textsize=scriptsize]{todonotes}

\newcommand{\rob}[1]{}

\usepackage{mdframed} 
\usepackage{thmtools}

\definecolor{shadecolor}{gray}{0.90}
\declaretheoremstyle[
headfont=\normalfont\bfseries,
notefont=\mdseries, notebraces={(}{)},
bodyfont=\normalfont,
postheadspace=0.5em,
spaceabove=1pt,
mdframed={
  skipabove=8pt,
  skipbelow=8pt,
  hidealllines=true,
  backgroundcolor={shadecolor},
  innerleftmargin=4pt,
  innerrightmargin=4pt}
]{shaded}

\declaretheorem[style=shaded,within=section]{definition}
\declaretheorem[style=shaded,sibling=definition]{theorem}
\declaretheorem[style=shaded,sibling=definition]{proposition}
\declaretheorem[style=shaded,sibling=definition]{assumption}
\declaretheorem[style=shaded,sibling=definition]{corollary}

\declaretheorem[style=shaded,sibling=definition]{lemma}